\documentclass[accepted]{uai2022} % after acceptance, for a revised
\usepackage[american]{babel}
\usepackage{natbib} % has a nice set of citation styles and commands
    \renewcommand{\bibsection}{\subsubsection*{References}}
\usepackage{mathtools} % amsmath with fixes and additions
\usepackage{booktabs} % commands to create good-looking tables
\usepackage{tikz} % nice language for creating drawings and diagrams
\usepackage{microtype}
\usepackage{graphicx}
\usepackage{subcaption}
\usepackage{algorithmicx}
\usepackage{algorithm}
\usepackage{algpseudocode}
\usepackage{times,enumitem,hyperref,multicol,blindtext}

\usepackage{amsthm}
\usepackage[T1]{fontenc}
\usepackage{amsmath}
\usepackage{amssymb}
\usepackage{amsfonts}
\usepackage{multirow}
\usepackage{mathtools}
\usepackage{breqn}
\usepackage{bbm}
\usepackage{dsfont}
\usepackage{graphicx}
\usepackage{natbib}
\usepackage{cases}
\usepackage[colorinlistoftodos]{todonotes}

\newcommand{\calens}{ID-calibrated ensembles}
\newcommand{\Calens}{ID-calibrated ensembles}

\newcommand{\natshifts}{natural shifts}
\newcommand{\natshift}{natural shift}
\newcommand{\adv}{anticorrelated spurious}
\newcommand{\advshifts}{anticorrelated spurious shifts}
\newcommand{\advshift}{anticorrelated spurious shift}

\newcommand{\Advshifts}{Anticorrelated spurious shifts}

\newcommand{\numidnat}{9}
\newcommand{\numnat}{11}
\newcommand{\numadv}{3}
\newcommand{\numtotal}{14}

\newcommand{\stdaccid}{88.6}
\newcommand{\stdaccood}{64.3}
\newcommand{\stdaccidnatural}{88.7}
\newcommand{\stdaccoodnatural}{65.2}
\newcommand{\robaccid}{86.9}
\newcommand{\robaccood}{74.6}
\newcommand{\robaccidnatural}{86.8}
\newcommand{\robaccoodnatural}{72.3}
\newcommand{\calaccid}{90.0}
\newcommand{\calaccood}{74.7}
\newcommand{\calaccidnatural}{90.3}
\newcommand{\calaccoodnatural}{74.5}
\newcommand{\tunedaccood}{72.1}

\newcommand{\naiveaccid}{89.4}
\newcommand{\naiveaccood}{73.1}

\newcommand{\stdeceid}{1.6}
\newcommand{\stdeceood}{11.3}
\newcommand{\robeceid}{2.3}
\newcommand{\robeceood}{6.8}
\newcommand{\robrobaccidSeven}{89.7}
\newcommand{\robrobaccoodSeven}{76.2}
\newcommand{\stdstdaccidSeven}{90.7}
\newcommand{\stdstdaccoodSeven}{68.8}
\newcommand{\calaccidSeven}{91.8}
\newcommand{\calaccoodSeven}{76.5}

\newcommand{\conf}{\mbox{conf}_\mathsf{id}}

\newcommand{\jens}{j_\mathsf{ens}}
\newcommand{\jstd}{j_\mathsf{std}}
\newcommand{\jrob}{j_\mathsf{rob}}

\newcommand{\pred}{\mbox{pred}}
\newcommand{\Err}{\mbox{Err}}

\newcommand{\Errid}{\mbox{Err}_\mathsf{id}}
\newcommand{\Errood}{\mbox{Err}_\mathsf{ood}}

\newcommand{\xval}{x^\mathsf{val}}
\newcommand{\yval}{y^\mathsf{val}}

\newcommand{\fstd}{f_\mathsf{std}}
\newcommand{\frob}{f_\mathsf{rob}}
\newcommand{\fstdbar}{\overline{f}_\mathsf{std}}
\newcommand{\frobbar}{\overline{f}_\mathsf{rob}}
\newcommand{\fens}{f_\mathsf{ens}}

\newcommand{\Tstd}{T_\mathsf{std}}
\newcommand{\Trob}{T_\mathsf{rob}}

\newcommand{\nval}{n_\mathsf{val}}

\newcommand{\Pid}{P_\mathsf{id}}
\newcommand{\Pood}{P_\mathsf{ood}}
\newcommand{\Padv}{P_\mathsf{adv}}

\renewcommand{\hat}{\widehat}

\newtheorem{assumption}{Assumption}[section]
\newtheorem{proposition}{Proposition}[section]

\newtheorem{lemma}{Lemma}[section]
\newtheorem{definition}{Definition}[section]
\newtheorem{corollary}{Corollary}[section]

\newtheorem*{remark*}{Remark}

\newtheorem*{observation*}{Observation}

\numberwithin{equation}{section}

\newcommand{\E}{\mathop{{}\mathbb{E}}}

\newcommand{\R}{\mathbb{R}}

\newcommand{\cX}{\mathcal{X}}
\newcommand{\cY}{\mathcal{Y}}

\newcommand{\cZ}{\mathcal{Z}}

\newcommand{\argmin}{\arg \min}
\newcommand{\argmax}{\arg \max}

\newcommand{\Gnorm}[1]{{\left\vert\kern-0.25ex\left\vert\kern-0.25ex\left\vert #1 
		\right\vert\kern-0.25ex\right\vert\kern-0.25ex\right\vert}}
\newcommand{\gnorm}[1]{{\vert\kern-0.25ex\vert\kern-0.25ex\vert #1 
		\vert\kern-0.25ex\vert\kern-0.25ex\vert}}

\newcommand{\softmax}{\mbox{softmax}}

\def\shownotes{0}
\ifnum\shownotes=1
\newcommand{\authnote}[2]{[#1: #2]}
\else
\newcommand{\authnote}[2]{}
\fi
\newcommand{\ar}[1]{{\color{orange}\authnote{AR}{#1}}}

\newcommand{\pl}[1]{{\color{red}\authnote{PL}{#1}}}
\newcommand{\ak}[1]{{\color{green}\authnote{AK}{#1}}}
\newcommand{\tnote}[1]{{\color{blue}\authnote{TM}{#1}}}

\title{Calibrated ensembles can mitigate accuracy tradeoffs under distribution shift}

\author[1]{\href{mailto:<ananya@cs.stanford.edu>?Subject=Your UAI 2022 paper}{Ananya Kumar}{}}
\author[1]{Tengyu Ma}
\author[1]{Percy Liang}
\author[2]{Aditi Raghunathan}
\affil[1]{%
    Computer Science Dept.\\
    Stanford University\\
    Stanford, California, USA
}
\affil[2]{%
    Computer Science Dept.\\
    Carnegie Mellon University\\
    Pittsburgh, Pennsylvania, USA
} 
  
\begin{document}
\maketitle

\begin{abstract}
  We often see undesirable tradeoffs in robust machine learning where out-of-distribution (OOD) accuracy is at odds with in-distribution (ID) accuracy:
a robust classifier obtained via specialized techniques such as removing spurious features often has better OOD but worse ID accuracy compared to a standard classifier trained via ERM.
In this paper, we find that \calens{}---where we simply ensemble the standard and robust models after calibrating on only ID data---outperforms prior state-of-the-art (based on self-training) on both ID and OOD accuracy.
On eleven natural distribution shift datasets, \calens{} obtain the best of both worlds: strong ID accuracy \emph{and} OOD accuracy.
We analyze this method in stylized settings, and identify two important conditions for ensembles to perform well both ID and OOD: (1) we need to calibrate the standard and robust models (on ID data, because OOD data is unavailable), (2) OOD has no anticorrelated spurious features.

\end{abstract}

\section{Introduction}
\label{sec:intro}

Machine learning models suffer large drops in accuracy out-of-distribution (OOD) where the test distribution is different from the training distribution.
For example, models trained on medical data from a few hospitals work poorly when deployed broadly~\citep{zech2018radio, albadawy2018tumor}. Similarly, when predicting poverty from satellite imagery, models trained on data from a few countries work poorly on new countries, particularly those where labels are scarce due to resource constraints~\citep{jean2016combining}.
There has been a lot of research interest in tackling this robustness problem under various settings such as robustness to spurious correlations~\citep{heinze2017conditional, sagawa2020group}, domain generalization~\citep{arjovsky2019invariant, sun2016deep}, demographic shifts~\citep{hashimoto2018repeated, duchi2019distributionally} among others.

\begin{figure*}[th]
    \centering
    \includegraphics[width=\textwidth]{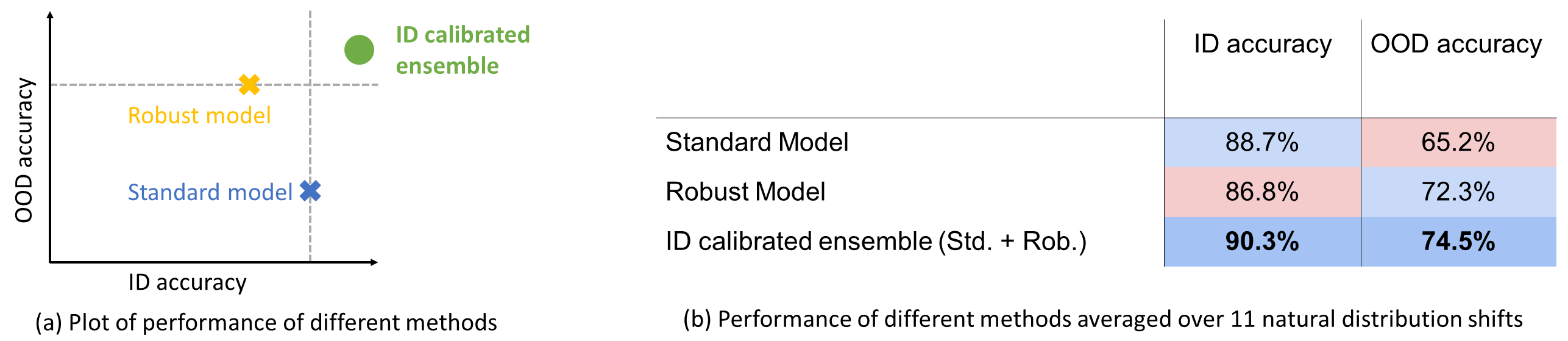}
    \caption{
      In many settings, we have a standard model that performs better in-distribution, and a robust model that performs better out-of-distribution.
      Across \numnat{} natural distribution shifts, ID-calibrated ensembles get the best of both worlds: the strong ID accuracy of the standard model and OOD accuracy of the robust model.
      We analyze its strengths and limitations in Section~\ref{sec:analysis}---as predicted by our analysis, ID calibrated ensembles do not perform as well on adversarially synthesized shifts with ``anticorrelated'' spurious features.
      We show full experimental results and ablations in Section~\ref{sec:experiments}.
    }
    \label{fig:calibration-figs}
\end{figure*}
\pl{I think this is too strong - the best robustness interventions these days - pre-training - improve both}
\ak{Agreed that it's too strong, so I've softened. But note that pretraining typically does have the same tradeoffs, e.g., in Lisa's prefix tuning, or LP-FT. We aren't comparing whether pretraining improves over no pretraining. We're comparing the best method for ID vs. the best method for OOD, so that'd be pretraining with LP vs. pretraining with FT. And we include such experiments in this paper!}
Across many of these settings, an \emph{unfortunate tradeoff} arises~\citep{tsipras2019robustness,xie2021innout,radford2021clip,li2021prefix,kumar2022finetuning}: robustness interventions, such as removing spurious features or lightweight fine-tuning, typically \emph{improve the out-of-distribution (OOD) accuracy but cause a drop in the in-distribution (ID) accuracy} on new test points from the original distribution. 
This tradeoff is a major hurdle in using robust models---in practice most inputs are likely to be ID, so it is unsatisfactory to use a robust model that performs less accurately on these majority ID points.  On the other hand, standard models (trained without robustness interventions) can fail in the presence of even small shifts, and it can be dangerous to use a standard model even if OOD points are rare. 
In this work, we ask: \emph{is there a general strategy to harness the strengths of both the standard and robust model to achieve high accuracy both ID and OOD, without using OOD data?}
\ar{I wonder if we should also say something about whether we even expect to mitigate this tradeoff at all (for e.g. the robust features stuff in Tsipras et al and many other papers says it's fundamental) ... and also it seems like we are obviously missing the seld training references.. perhaps we can kill two birds by citing self training works that show that it is possible to mitigate the tradeoff but they require a large amount of unlabeled data. And then add the qualifier ``without any additional data from the target domain'' in the question? It's certainly less cleaner than what you have, so think about it} 
\tnote{i think a middle option is to cite Tsipras et al to demonstrate that best of both worlds should not be taken for grant, but not mention the self-training works to make the narrative simpler}
\ak{These are good idea, and I'm thinking about how to work them in. The current phrasing does ask whether we can harness the strengths of both (doesn't take it for granted). Citing Tsipras et al is a good idea.}

We find that \calens{}, a simple approach of first calibrating the standard and robust models on only ID data and then ensembling them, outperforms prior state-of-the-art both ID and OOD.
As illustrated in Figure~\ref{fig:calibration-figs}, across \numnat{} natural distribution shift datasets (e.g. geographical shift, style shift, subpopulation shift), \calens{} get the \emph{best of both worlds}: better ID and OOD accuracies than both the standard and robust models.
Averaged across these datasets, \calens{} achieve an ID accuracy of \calaccidnatural{}\% (vs. \stdaccidnatural{}\% for the standard model and \robaccidnatural{}\% for the robust model) and OOD accuracy of \calaccoodnatural{}\% (vs. \stdaccoodnatural{}\% for the standard model and \robaccoodnatural{}\% for the robust model).

To build intuitions for when and why ID-calibrated ensembles can get the best of the standard and robust models, we analyze a toy setting where these models produce independent signals for the label.
\pl{why would we expect the signals to be independent? I wouldn't expect them to be}
\ak{Yeah, they probably aren't (although works such as simplicity bias could motivate this---e.g., ERM often exclusively uses spurious features, robust methods project out spurious features and use others). I think it would be good to analyze more general settings in the future!}
The first step of our method is to calibrate the standard and robust models on ID data.
We show that after this calibration step, the ensembling strategy for the best ID performance is to simply add \pl{doesn't always type check, say combine?}\ak{I'm wondering if combine sounds vague. Also, why doesn't add typecheck?} the predictions of the two models (Proposition~\ref{prop:calibration-ensemble-optimal}).
In particular, this approach (\calens{}) outperforms both the standard and robust models on ID test examples.

When can \calens{} provide benefits OOD even though it does not use any OOD data?
In many natural distribution shifts, standard models pick up on predictive signals in the training data that are absent or suppressed under distribution shift---in these cases, we show that \calens{} obtain the best of both the standard and robust models OOD.
However, when spurious features become anticorrelated OOD (as is common when the distribution shift is adversarially synthesized), we show that the ensemble's OOD accuracy is in between the standard and robust models.
We empirically validate this on three adversarially synthesized shifts~\citep{sagawa2020group,jones2021selective} where the spurious signals are anticorrelated OOD.

We find that \calens{} outperform prior approaches based on self-training~\citep{carmon2019unlabeled,uesato2019are,xie2021innout}, despite not using any additional unlabeled data.
Finally, we compare \calens{} to a number of other ensembling strategies (for example, tuning the weights of the ensemble on ID validation data) and find that they do not work as well as \calens{}.
\pl{other ensembles don't calibrate? maybe make that clearer?}
\ak{doesn't make a difference for tuned ensembles (formally depends on the variant, whether we do logits or probs, but in practice doesn't matter)}

To summarize, our main contributions are:
\begin{enumerate}
  \item We revisit the classic idea of ensembling and propose a simple, general, and effective method (\calens{}) to mitigate ID-OOD accuracy tradeoffs (without using OOD data).
  \item \calens{} outperform prior approaches based on self-training, despite not using any additional unlabeled data.
% and other specialized approaches that only work for specific kinds of shifts and tradeoffs. 
  \item We find that \calens{} eliminate tradeoffs under a variety of natural distribution shifts, but can fail when there are adversarially synthesized shifts.
\end{enumerate}

\ar{Overall, the intro looks good to me. Similarly to the comment on the abstract, I wonder if we should lead with the experimental result that CEs match or beat self-training. I feel the method/analysis is not that exciting on its own, it's exciting because it works so well. So starting with that observation seems more compelling to me, though it breaks the usual theory -> experimetns flow of a standard paper}
\ak{Changed the abstract to do this, will think about the intro (I like that it's clean now, but agreed that would be good to talk about self-training since that's the most interesting result)}

\section{Setup}
\label{sec:setup}

Consider a $K$-class classification task, where the goal is to predict labels $y \in [K]$ corresponding to inputs $x \in \cX$.

\textbf{Models.}
A model $f: \cX \to \R^K$ takes an input $x \in \cX$ and outputs a score $f(x) \in \R^K$ where $f(x)_i$ can be interpreted as the model's ``confidence'' that the label $y$ is $i$.
The model outputs the label $\pred(f(x)) = \argmax_i f(x)_i$. 
The confidence scores can be normalized to sum to $1$ (and interpreted as probabilities) using the softmax function, $\softmax(f(x))_i = \frac{\exp(f(x)_i)}{\sum_{j=1}^K \exp(f(x)_j)}$ for $i \in [K]$.

\textbf{Distributions and error.}
Let $\Pid$ and $\Pood$ denote the underlying distribution of $(x, y)$ pairs in-distribution (ID) and out-of-distribution (OOD), respectively.
We evaluate a model $f$ on the fraction of times it makes a wrong prediction on $\Pid$ and $\Pood$: $\Errid(f) = \E_{x, y \sim \Pid}[ \pred(f(x)) \neq y]$ and $\Errood(f) = \E_{x, y \sim \Pood}[ \pred(f(x)) \neq y]$.
% \ar{I think we can just define $\Errid$ and $\Errood$ directly... feels easier that going via the abstract $\Err(P, \cdot)$ and then also compressing the notation}
% \ak{done}
% \ar{How exactly are the standard and robust models defined? I think standard should be defined as ERM model and robust model is any model that performs better than the ERM model OOD? Right now, it's unclear and confusing}
% \ak{added}

\textbf{Standard and robust models.}
A standard model $\fstd$ is trained via empirical risk minimization (ERM) where we minimize some loss on ID training data.
$\fstd$ often might rely on spurious correlations between the image and label such as image background or occurence of certain words that are not necessarily predictive OOD.\ak{TODO: add cites}
% Hence standard models often perform poorly OOD.
% In order to improve OOD performance, the training process needs to be changed (robustness interventions) to discourage models from relying on ID-specific spurious features.
% We call such models $\frob$, where the exact robustness intervention depends on the task. 
In order to improve OOD performance, a robust model $\frob$ is trained via a modified training procedure (robustness interventions) to discourage models from relying on ID-specific spurious features.
We have the following \emph{tradeoff} between $\fstd$ and $\frob$.
\begin{align}
\Errid(\fstd) \leq \Errid(\frob); ~~\Errood(\frob) \leq \Errood(\fstd). 
\end{align}
\pl{this is assuming infinite data...I think in general we need to be clear about this that we're not thinking about generalization?}
\ak{Trying to understand this better---are you saying the standard and robust models might not satisfy this property with finite data? E.g., in In-N-Out if we have very little data, robust model can do better since it uses fewer features? We're just taking this tradeoff as a given though here. Agreed that we should say we don't really look at finite samples in the analysis.}
The precise robustness intervention depends on the task---in Section~\ref{sec:analysis} we model the relationship between $\fstd$ and $\frob$ in a stylized setting amenable for analysis, and in Section~\ref{sec:datasets} we describe what $\fstd$ and $\frob$ are in our real datasets.
% We call such models $\frob$, where the exact robustness intervention depends on the task. 
\ar{Is there an easy example to give here?InNOut and LP/FT both seem a bit difficult to explain intuitively?}
\ak{Thinking---maybe I can just say by projecting out spurious features}

\textbf{Best of both worlds.} Our goal is to get the best of both worlds---a classifier $\fens$ that achieves better ID accuracy than the standard model, and better OOD accuracy than the robust model.
\begin{align}
\Errid(\fens) \leq \Errid(\fstd); ~~\Errood(\fens) \leq \Errood(\frob). 
\end{align}

% \begin{equation}
% 	\Err(f, P) = \E\limits_{x, y \sim P}[ \pred(f(x)) \neq y]
% \end{equation}
% where $\Errid(f) = \Err(f, \Pid)$ and $\Errood(f) = \Err(f, \Pood)$.
% \Err(f, P) = \E\limits_{x, y \sim \Pid}[ \pred(f(x)) \neq y]\mbox{, and }\Errood(f) = \E\limits_{x, y \sim \Pood}[ \pred(f(x)) \neq y],
% , or fraction of misclassifications on \ar{samples from} a distribution. Formally, for a distribution $P$, we have $\Err(P, \cdot) = \E\limits_{x, y \sim P}[ \pred(f(x)) \neq y]$. Let $\Pid$ and $\Pood$ denote the underlying distribution of $(x, y)$ pairs in-distribution (ID) and out-of-distribution (OOD), respectively.
% In this work, we evaluate models on $\Err(\Pid, \cdot)$ and $\Err(\Pood, \cdot)$ and $\Pood$, denoted by $\Errid$ and $\Errood$ respectively. We measure $\Errid$ and $\Errood$ on held-out test sets drawn from $\Pid$ and $\Pood$ respectively. 
% 
% We have a standard model $\fstd$ and a robust model $\frob$, with a tradeoff: the standard model typically does better ID while the robust model does better OOD: $\Errid(\fstd) > \Errid(\frob)$ but $\Errood(\fstd) < \Errood(\frob)$---our goal is to get the best of both and produce a model that does well both ID and OOD. \ar{Make precise..? Setup should be precise - this sounds like an intro; Either remove this whole part and just have the ID and OOD errors without going into standard and robust models, or make everything more concrete by grounding in terms of ERM}

\textbf{ID validation data.}
% We have training data from $\Pid$, $\{(\xtrain_i, \ytrain_i)\}_{i=1}^{\ntrain} \sim \Pid$.
% \ak{@Aditi, I know you suggested having training data, but I thought about it more and commented this out because we don't use the training data anywhere. Is that OK?}
To get the best of both worlds, we only allow access to ID validation data, $\{(\xval_i, \yval_i)\}_{i=1}^{\nval} \sim \Pid$, for tuning hyperparameters.
Following~\citet{xie2021innout,koh2021wilds,gulrajani2020search} we do \emph{not} use any OOD validation data.
% ---this captures the setting where it is difficult to predict what kind of distribution shifts occur when a model is deployed. 
\ak{Do we need to explain why we don't have OOD data?}
% \ar{Be consistent across period or colon after titles. I think currently it reflects whether you wrote something or i did :) }
% \ak{Changed all to periods :)}
% \ar{I think active voice would be stronger... in this work, we only allow access to ID validation set and maybe take the chance to say that most other works assume some access to OOD information---either validation set or unlabeled data?}
% \ak{Changed to active voice! I agree with you that not using OOD data is a strength of our work, but I want to be careful that we don't confuse the reader, and be mindful there are works that only use ID data like In-N-Out. Any suggestions on what to add here?}

\pl{but now we are in the finite sample regime?}
\ak{I think sample complexity isn't too critical here (but could be wrong), do you think we should explicitly spell out that we have finite samples / does that add value to the reader?}
\pl{I think it'd be clarifying to specify what fstd and frob and fens can depend on...because we don't talk about the training set}
\ak{Ok, will think about how to frame this. We define things in the analysis and experiments properly, I think our experiments are too broad to easily capture here. But I agree that it would be good to, let me know if you have thoughts!}

% In addition, we have a validation set $\{(\xval_i, \yval_i)\}_{i=1}^{\nval} \sim \Pid$ that can be used to tune hyperparameters.  Note that the validation set is exclusively from $\Pid$ since OOD data is typically unavailable.
% and we do \emph{not} use any OOD validation data---this captures the practical setting where we cannot accurately predict when and and what kind of distribution shifts occur when a model is deployed. 

\section{Methods}
\label{sec:methods}
% \ar{I feel this section should go before intuition and analysis}
% \ak{Done}

% \begin{algorithm*}[tbp]
% 	\caption{\calens{}}	\label{alg:calens}
% 	\begin{algorithmic}[1]
% 		\small
% 				\Require  in-distribution validation data $\{(\xval_i, \yval_i)\}_{i=1}^{\nval} \sim \Pid$, 
% 		\Statex	\hspace{0.7cm} standard and robust models $\fstd, \frob : \cX \to \R^K$
% \State Calibrate $\fstd$ on ID data: $\Tstd = \argmin_T \frac{1}{\nval} \sum_{i=1}^{\nval} l\Big(\frac{\fstd(\xval_i)}{T}, \yval_i\Big)$ \;
% \State Calibrate $\frob$ on ID data: $\Trob = \argmin_T \frac{1}{\nval} \sum_{i=1}^{\nval} l\Big(\frac{\frob(\xval_i)}{T}, \yval_i\Big)$ \;
% \State Return $\fens$ where $\fens(x) = \fstd(x) / \Tstd + \frob(x) / \Trob$
% 	\end{algorithmic}
% \end{algorithm*}

\begin{algorithm}[tbp]
	\caption{\calens{}}	\label{alg:calens}
	\begin{algorithmic}[1]
		\small
				\Require  in-distribution validation data $\{(\xval_i, \yval_i)\}_{i=1}^{\nval} \sim \Pid$, 
		\Statex	\hspace{0.7cm} standard and robust models $\fstd, \frob : \cX \to \R^K$
\State $\fstdbar$ = Calibrate $\fstd$ on in-distribution (ID) data \;
\State $\frobbar$ = Calibrate $\frob$ on in-distribution (ID) data \;
\State Return $\fens(x) = \log\big(\softmax(\fstdbar(x)) + \softmax(\frobbar(x))\big)$
	\end{algorithmic}
\end{algorithm}

\textbf{Proposed method: \calens{}.}
Given a standard model $\fstd$ and robust model $\frob$, we first calibrate each model on the \emph{in-distribution} validation data, and then add up their predictions (Algorithm~\ref{alg:calens}).
In our experiments, we calibrate using a variant~\citep{luo2020privacy} of temperature scaling~\citep{guo2017calibration} that was found to work better (Appendix E.2 of ~\citep{luo2020privacy}).

Neural networks are often overconfident in their predictions~\citep{guo2017calibration}, so the aim of the calibration step is to match up the models' confidences with their accuracies.
% The average accuracy of a model $f$ is simply the proportion of ID validation examples $f$ gets correct.
% % This does not depend on the temperature $T$, because changing the temperature does not change the model's prediction (the $\argmax$).
% \begin{align}
% \acc(f) = \frac{1}{\nval} \sum_{i=1}^{\nval} \mathbbm(\argmax_j f(\xval_i)_j = \yval_i)
% \end{align}
The average confidence, over the ID validation set, of a model $f$ scaled by a temperature $T$ is given by the average of the model's probability for its prediction (the prediction is the $\argmax$ and so the probability is the $\max$).
\begin{align}
\conf(f, T) = \frac{1}{\nval} \sum_{i=1}^{\nval} \max_j \softmax\Big(\frac{f(\xval_i)}{T}\Big)_j
\end{align}
To calibrate, we choose $\Tstd$ and $\Trob$ such that the standard and robust models' confidences match up with their accuracies.
We implement this with binary search, which works since the confidence increases when $T$ decreases.
\begin{align}
\label{eqn:tstd_trob_eqn}
\conf(\fstd, \Tstd) &\approx 1-\Errid(\fstd) \mbox{ and, } \\
\conf(\frob, \Trob) &\approx 1 - \Errid(\frob)
\end{align}

After calibration, we simply ensemble the two models by averaging the probabilities\footnote{The purpose of the $\log$ is to convert back to logit space for $\fens$. The equation looks like a sum and doesn't have a $\frac{1}{2}$, but is equivalent to averaging the probabilities after applying $\softmax$ because $\softmax$ normalizes into probabilities. While it might seem a bit strange to average probabilities, we found this to be a bit more reliable at mitigating tradeoffs than adding the logits (multiplying the probabilities) OOD---see Tables~\ref{tab:ood_results} and ~\ref{tab:ood_tuned}.} that they predict~\citep{lakshminarayanan2017simple}.
\ak{Note: adding and averaging are the same, since we convert back to logit space, but I don't want to confuse potential readers with that.}
\pl{I would have expected geometric average...why does this make sense?}
\ak{That's a great question! Not entirely sure why it makes sense. Prior work does this, and I compared with a few other things (e.g., adding logits), and this seems to work better. I think geometric average (basically multiplying probabilities) is similar to adding logits, since logits are in log space, but will check this intuition!}
% \begin{equation}
% 	\fens(x) = \frac{1}{2}\Big( \softmax\Big(\frac{\fstd(x)}{\Tstd}\Big) + \softmax\Big(\frac{\frob(x)}{\Trob}\Big)\Big),
% \end{equation}
\begin{equation}
	\fens(x) = \log\Big(\softmax\Big(\frac{\fstd(x)}{\Tstd}\Big) + \softmax\Big(\frac{\frob(x)}{\Trob}\Big)\Big),
\end{equation}
where the predicted label is $\pred(\fens(x)) = \argmax_y \fens(x)_y$, and the predicted probabilities are $\softmax(\fens(x))$.
% \ar{Should mention this is our proposed method, and perhaps acknowledge in intro that this is very intuitive but we didn't find any reference to a paper that uses calibrated ensembles}
% \ak{added proposed method, will check intro}

% The goal of temperature scaling is to adjust each model's confidence on the in-distribution validation data.

\paragraph{Ablations.}
In Section~\ref{sec:experiments} we ablate each component of the method, for example the calibration step, way of combining the models, and we compare to (calibrated) ensembles of two standard models, or of two robust models.

\section{Intuitions and analysis}
\label{sec:analysis}

In this section, we build basic intuitions for when and why \calens{} can get the best of both worlds (good ID accuracy of $\fstd$ and OOD accuracy of $\frob$), even without using any OOD data.
We first define a stylized setting, and then analyze the ID performance in Section~\ref{sec:analysis_id} and OOD performance in Section~\ref{sec:analysis_ood}.
While the analysis is stylized, the key strength of \calens{} is the strong \emph{empirical} performance on a wide range of realistic datasets, robustness interventions, and modalities, in Section~\ref{sec:experiments}.
% The key selling point is the strong empirical performance in Section~\ref{sec:experiments}.
\ak{Maybe say analyzing ensembling is a big open problem in theoretical machine learning, which is why we need to resort to such a stylized setting. Distribution shifts only increase this problem.}
% Our goal is to get the best of both worlds (good ID accuracy of $\fstd$ and OOD accuracy of $\frob$) using only ID validation data.
% Here, we consider a stylized setting to build up to a principled approach.
% We analyze the ID accuracy in Section~\ref{sec:analysis_id} and OOD accuracy in Section~\ref{sec:analysis_ood}.
% We describe our final method in Section~\ref{sec:methods}.
% Standard models are trained with empirical risk minimization on training data, while robust models are trained with a robustness intervention such as projecting out spurious features [cites], linear probing a pretrained model [cites], zero-shot language prompting [cites], and distributionally robust optimization [cites].
% and evaluate on real datasets in Section~\ref{sec:experiments}.

% we assume $\fstd$ relies on spurious features of the input (that) 
% Recall that the goal of this work is to combine the strengths of a standard model $\fstd$ (with lower ID error) and a robust $\frob$ (with lower OOD error) in order to achieve the best of both worlds---low ID and low OOD error. To do so, we need to understand how $\fstd$ and $\frob$ differ.
% To get the best of $\fstd$ and $\frob$, we need to understand the relationship between $\fstd$ and $\frob$.
\textbf{Conditional independence.}
The literature on simplicity bias shows that standard models trained via empirical risk minimization (ERM) often exclusively rely on simple spurious patterns, whereas robust models are trained to avoid these patterns.
Motivated by this, we assume inputs have some robust features (that are predictive both ID and OOD) and some spurious features (that are only predictive ID).
$\fstd$ relies on the spurious features while $\frob$ relies on the robust features, both of which provide independent signals on the label.

\begin{assumption}
  We assume that $\frob$ and $\fstd$ have conditionally independent outputs with respect to $\Pid$ and $\Pood$, that is,
\begin{equation}
\frob(x) \perp \fstd(x) \mid y \quad \mbox{when $(x, y) \sim P$ for $P \in \{\Pid, \Pood\}$}
\end{equation}
\end{assumption}

\ar{Connect with the previous...maybe start with this line and then say ``in fact, it's wearker...}
\ak{Sorry, didn't get this suggestion, are you suggesting adding ``in fact'' to the start of the paragraph?}
\textbf{Connection with prior assumptions.}
Our assumption that the \emph{model outputs} are conditionally independent is weaker than assumptions in prior conceptual models of distribution shifts~\cite{chen2020selftraining,sagawa2020overparameterization,nagarajan2020understanding} where robust and spurious features are disjoint parts of the input, each generated independently based on the label. In our setting, the features can be complicated functions of the inputs.
% (that are conditionally independent given the label).

\textbf{Ensemble.} The ensemble $\fens$ simply adds up the predictions of the standard model $\fstd$ and robust model $\frob$. This is slightly different from Section~\ref{sec:methods}, where we add up the probabilities of the model instead of the predictions/logits, but is more amenable to analysis.
\begin{equation}
    \label{eqn:ensemble_dfn}
    \fens(x) = \fstd(x) + \frob(x)
\end{equation}
\pl{given that we average in 3.3, why don't we do that here instead of sum to make things consistent?}
\ak{In the algorithm box we take the average of the probabilities, and it seems strange to add them up or they won't be probabilities. Here we add the logits, and it's necessary for the theory. I think it would be great to analyze these nuances more thoroughly in the future!}

% \ar{motivate why label distribution is important...this is not something people usually think about and is also not defined in the setup?}
% \ak{Saw this comment, trying to think about how to movivate it in a succinct way}
% \ak{update: discussed on slack, and edited accordingly}
\textbf{Class-balanced.} For simplicity of exposition, we assume the class-balanced setting where every label $P(Y=y)$ is equally likely. Formally, we say $P$ is class-balanced if $P(Y=y) = 1/K$ for all $y \in [K]$. We analyze the general setting in Appendix~\ref{app:analysis_appendix}.
\pl{we want to say this for both $\Pid$ and $\Pood$}\ak{We actually also need it for components of $\Pood$. So here I'm just defining what it means for arbitrary $P$, and then the statements of the results invoke the definition, e.g., assume $\Pid$ is class-balanced. Is that ok?}

\subsection{ID performance of ensembles}
\label{sec:analysis_id}

In this section, we show that if $\fstd$ and $\frob$ are \emph{calibrated} with respect to $\Pid$, then the ensemble $\fens$ is the best way to combine their predictions.
Since we have access to validation data from $\Pid$, the first step of our method (Section~\ref{sec:methods}) is to calibrate $\fstd$ and $\frob$ ID.
\pl{this is very confusing since you have mutation;
$\fstd$ and $\frob$ are given as input to your algorithm,
so you should compute the temperature scaling $T$;
but then it's not $\fstd$ that's calibrated, but rather $\fstd/T_{std}$...
}\ak{will think about how to improve this}
We conclude the section by giving intuition for why this calibration step can be particularly important for deep neural networks.
% We first examine the ID performance of the ensemble $\fens$, and in Section~\ref{sec:analysis_ood} we examine the OOD performance.
% We show that if $\fstd$ and $\frob$ are \emph{calibrated} with respect to $\Pid$, then the best way to combine them is to ensemble them: to add up their pedictions and output the label with highest confidence.

Intuitively, calibration means that the probability that a model outputs for an event reflects the true frequency of that event: if a model says 1,000 patients have the flu with probability 0.1, approximately 100 of them should indeed have the flu.
Formally, we look at joint calibration~\citep{murphy1973vector, brocker2009decomposition} where a model $f$ is calibrated with respect to a distribution $P$ if for all $x \in \cX, y \in [K]$:
\begin{equation}
P(y \mid f(x)) = \softmax(f(x))_y
% P(Y = y \mid f(X) = f(x)) = \softmax(f(x))_y
\end{equation}

The following proposition says that if $\fstd$ and $\frob$ are calibrated on $\Pid$, then $\fens$ has lower error on $\Pid$ than any other way of combining the two models---this also implies that $\fens$ gets higher accuracy than $\fstd$ and $\frob$. 
\newcommand{\calibrationEnsembleOptimalText}{
Suppose that $\fstd$ and $\frob$ are calibrated with respect to $\Pid$, and that $\Pid$ is class-balanced.
% Let $\fens(x) = \frob(x) + \fstd(x)$.
Let $h : \R^K \times \R^K \to \R^K$ be an arbitrary function that combines the standard and robust model's predictions, and let $f_h$ be the resulting classifier: $f_h(x) = h(\fstd(x), \frob(x))$.
The ensemble is better than any such combination classifier $f_h$: $\Errid(\fens) \leq \Errid(f_h)$.
}
\begin{proposition}
\label{prop:calibration-ensemble-optimal}
\calibrationEnsembleOptimalText{}
\end{proposition}

The proof of Proposition~\ref{prop:calibration-ensemble-optimal} is in Appendix~\ref{app:analysis_appendix}. Intuitively, since $\frob(x) \perp \fstd(x) \mid y$, the Bayes optimal predictor is proportional to multiplying their predicted probabilities, which is equal to adding logits (logits are in log space).
Proposition~\ref{prop:calibration-ensemble-optimal} has an important condition: the two models must be calibrated.
In practice, deep learning models are miscalibrated~\citep{guo2017calibration}, so our first step (Section~\ref{sec:methods}) is to calibrate the models ID.
We explain why the ID calibration step is important for deep neural networks.

\textbf{Why neural networks are miscalibrated.}
Deep neural networks are typically large enough to memorize the training dataset, and are encouraged to magnify their weights (and hence their confidence) to decrease the training loss~\citep{mukhoti2020calibrating,bai2021dont}.
The extent of this miscalibration and overconfidence depends on the training procedure~\citep{hendrycks2019pretraining,desai2020calibration}.
In our case $\fstd$ and $\frob$ are trained in different ways and have different calibration (Appendix~\ref{sec:per-dataset-calibration-appendix}).

\textbf{Why this miscalibration can hurt ensembling.}
% If we directly ensemble these miscalibrated models, ensembles may not get the best of both worlds.
Concretely, consider two models $\fstd'$ and $\frob'$ which are calibrated on $\Pid$.
Let $\fstd(x) = M \fstd'(x)$ for large $M \in \R$ (this magnifies its weights as discussed above), and let $\frob = \frob'$.
$\fstd$ and $\fstd'$ have the same predictions and therefore accuracy but $\fstd$ is highly miscalibrated.
The ensemble is then given by $\fens(x) = \fstd(x) + \frob(x) = M \fstd'(x) + \frob'(x)$.
For very large $M$, $\fens$ and $\fstd'$ have the same predictions---this means that $\Errood(\fens) = \Errood(\fstd) < \Errood(\frob)$, and so ensembling does not get the best of both worlds.
Note that if $\fstd$ and $\frob$ are miscalibrated by the same amount, then ensembling will still get the best of both worlds, but if one of the models is more miscalibrated than the other model OOD (which we see on a real dataset in Appendix~\ref{sec:per-dataset-calibration-appendix}) then ensembling can work poorly.

\subsection{OOD performance of ensembles}
\label{sec:analysis_ood}

So far, we showed that if $\fstd$ and $\frob$ are calibrated on a distribution $P$, then $\fens$ is better than both models on $P$.
However, our validation data is from $\Pid$, so we can only calibrate $\fstd$ and $\frob$ ID.
Even after this ID calibration step, $\fstd$ and $\frob$ are usually very miscalibrated OOD (on $\Pood$---see Appendix~\ref{sec:per-dataset-calibration-appendix} and~\citet{ovadia2019uncertainty}).

Our goal in this section is to build basic intuitions for when ID-calibrated ensembles can get high OOD accuracy.
% In general, analyzing how deep networks perform OOD is very challenging---
We draw inspiration from distribution shift benchmarks, defining simplified and stylized versions of those shifts.
A toy version of our analysis is visualized in Figure~\ref{fig:analysis_intuitions}, where the standard model relies on spurious features that change out-of-distribution.
If these features are ``suppressed'' or ``missing'' OOD, then $\fens$ does better than $\fstd$ and $\frob$ (Figure~\ref{fig:sup_spur}).
However, if these features are anticorrelated OOD (correlated with the opposite label) then the accuracy of $\fens$ is between $\fstd$ and $\frob$ (Figure~\ref{fig:adv_spur}).
We begin by formalizing these shifts, and then analyze the accuracy under these shifts.

\begin{figure*}
    
     \begin{center}
     \hfill
	 \begin{subfigure}[b]{0.25\textwidth}
	     \centering
	     \includegraphics[width=\textwidth]{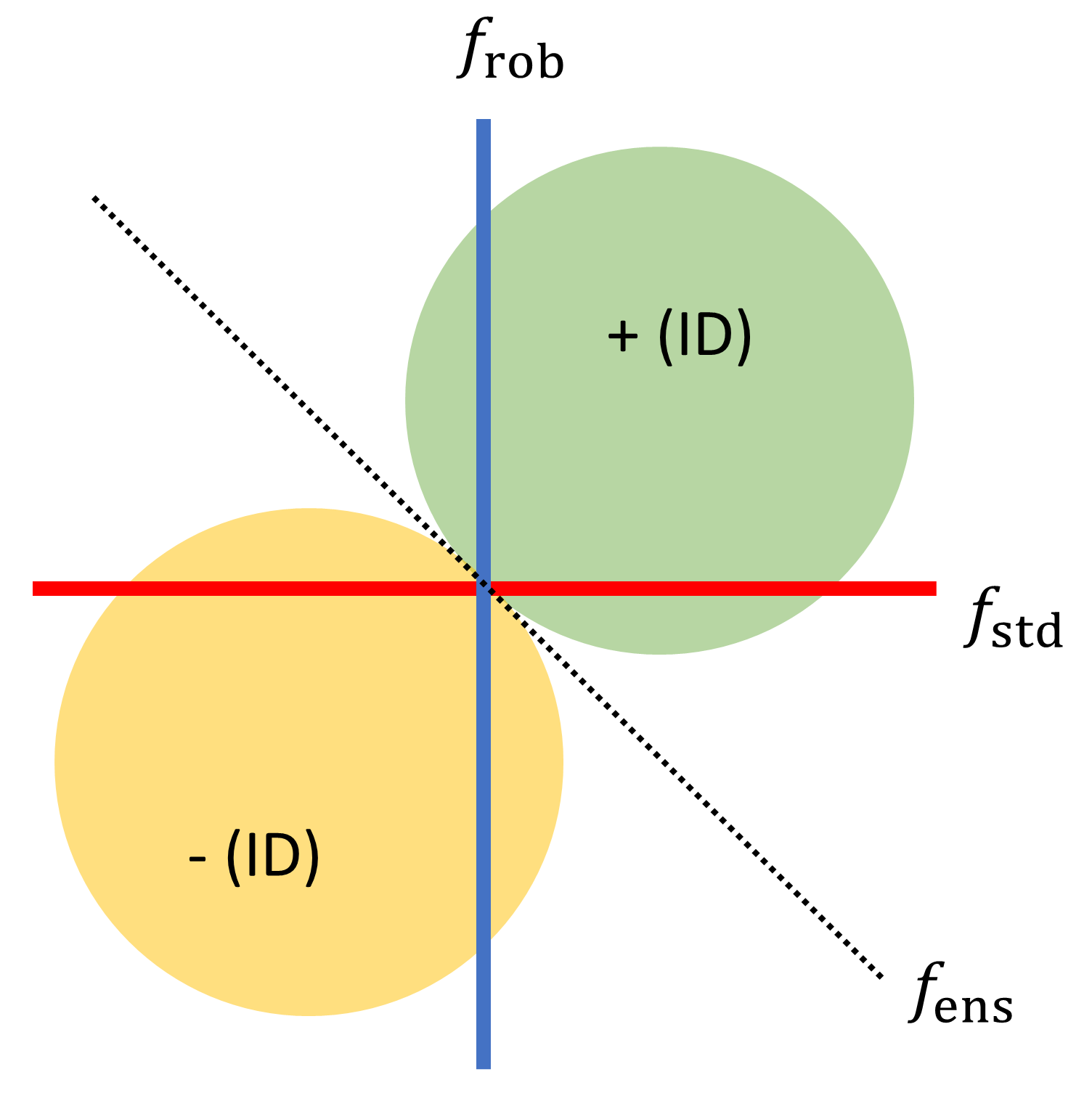}
	     \caption{In-distribution}
	     \label{fig:id_no_spur}
	 \end{subfigure}
     \hfill
     \begin{subfigure}[b]{0.25\textwidth}
         \centering
         \includegraphics[width=\textwidth]{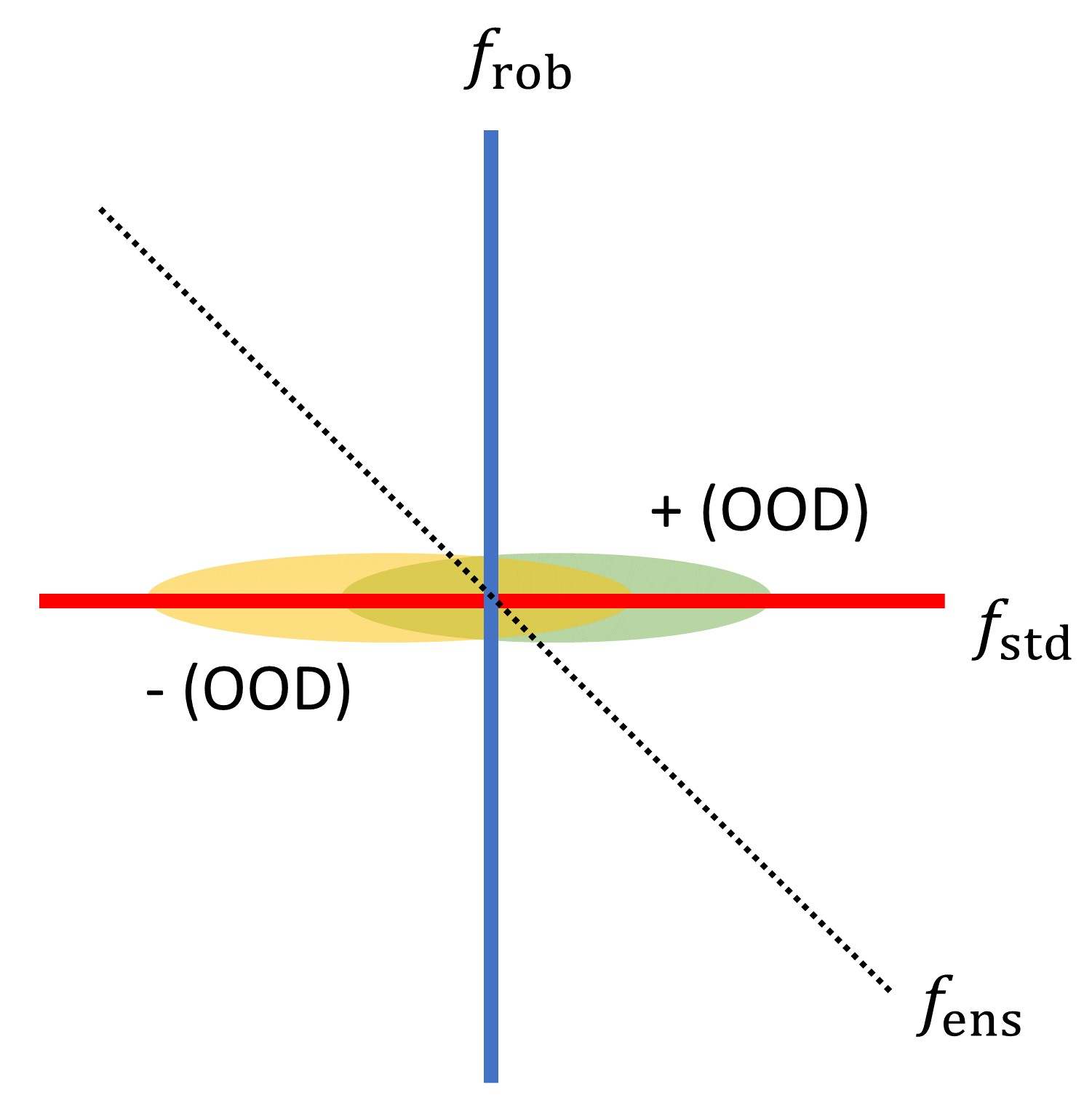}
         \caption{Missing spurious}
         \label{fig:sup_spur}
     \end{subfigure}
     \hfill
     \begin{subfigure}[b]{0.25\textwidth}
         \centering
         \includegraphics[width=\textwidth]{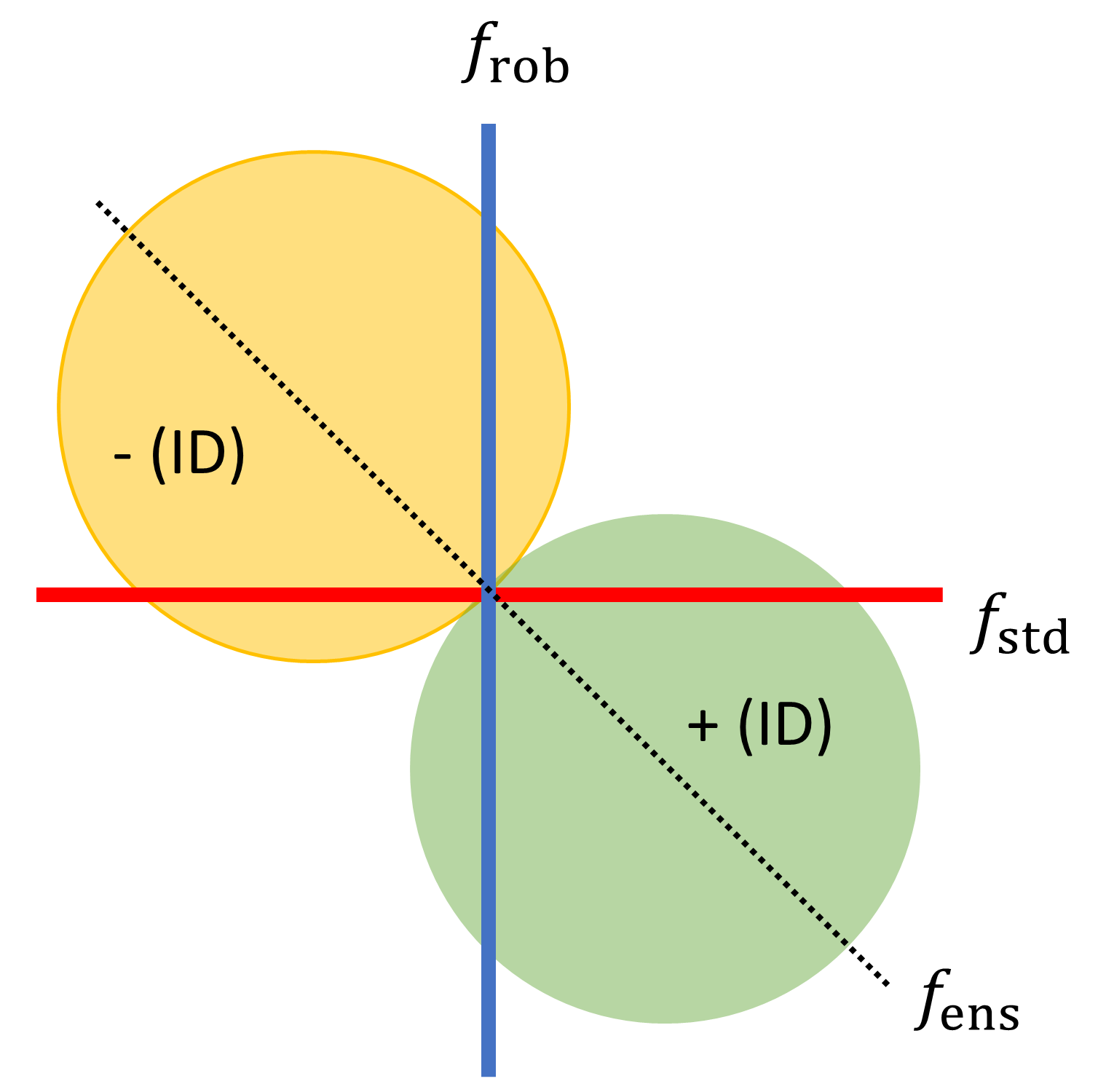}
         \caption{Anticorrelated spurious}
         \label{fig:adv_spur}
     \end{subfigure}
     \hfill
     % \hfill
     \caption{
     % A simple example where the data distribution shifts gradually over time.
     A toy version of our analysis in Section~\ref{sec:analysis}. (Figure~\ref{fig:id_no_spur}) Given a standard model $\fstd$ (red horizontal line) and robust model $\frob$ (blue vertical line) that use different aspects of the data, ensembling their predictions gives a predictor $\fens$ (black dotted line) with lower error---in this case $\fens$ completely separates the positive (green circle) and negative (yellow circle) examples in-distribution (ID). (Figure~\ref{fig:sup_spur}) $\fstd$ uses spurious features, suppose that these features are missing OOD (e.g., the $y$ component of the input goes close to $0$)---then $\fstd$ fares poorly and mislabels half the inputs, but the ensemble $\fens$ is about as accurate as the robust model $\frob$. (Figure~\ref{fig:adv_spur}) On the other hand, suppose the spurious features are \emph{anticorrelated} with the label OOD. In this case $\fens$ intersects the positive (yellow circle) and negative (green circle) distributions, and gets 50\% error---here $\fens$ is worse than $\frob$ but better than $\fstd$.
\ar{Figure can be made nicer and more self-explanatory. What exactly are you trying to convey in these figures actually? You should perhaps shade the misclassified regions of all the classifiers or atleast the ensemble. Also, you are assuming that the models are ID-calibrated so that should be mentioned. Regarding making figure nicer: reduce white space (shorten the axes), increase font size, the orange circles have a blue border - change that, increase thickness of lines. Another thing i noticed is that standard and robust are both equal accuracy ID. Can we change that to make standard model more robust? Also what are the spurious and robust features? They are not marked on the axes.}
     }
	\label{fig:analysis_intuitions}
	\end{center}
	\vskip -0.2in
\end{figure*}

% \subsubsection{Shifts in spurious features}
% \ak{Maybe say something like if you're calibrated OOD then you're good. But ID calibration does not imply OOD calibration (even in practice).}
% \ar{I realized that we have never defined formally what spurious and robust features are. Ideally, this should be folded into assumption 3.1. Or you have to have a setup here in this section about the format of stylized settings you are studying and mention spurious and robust features there.}
% We first describe the different types of shifts in spurious features and their implications for calibrated ensembles. We assume there is no label shift i.e. $\Pood(Y = y) = 1/k$. 

\textbf{Missing spurious.} For our first setting, we draw inspiration from some distribution shift benchmarks. Consider Breeds Living-17~\citep{santurkar2020breeds} where the goal is to classify an image as one of 17 animal categories. The category `bear' in the ID training data contains images of black bears and sloth bears while the OOD dataset has images of brown bears and polar bears.
A standard model trained on the ID dataset might latch onto very specific features about sloth bears (for example the presence of a shaggy mane) which are simply missing in the OOD dataset ($\fstd(x) = 0$).
A robust model could be trained to project these features out~\citep{xie2021innout}, so its predictions are still fairly reliable OOD.
% However, the robust features are still reliable OOD but potentially miscalibrated, leading to the following definition.
% \ar{What is reliable, but miscalibrated? You might want to connect with the $\alpha$ term after the definition below}
% \ak{That's a good point, I'm thinking about how to connect with $\alpha$, but a bit stuck on that}
\begin{definition}[missing spurious]
\label{dfn:missing_spurious}
  A distribution $P_0$ has missing spurious features if for $x \sim P_0$, we have $\fstd(x) = 0$ almost surely and there exists some $\alpha \in \R^+$, such that for all $x \in \cX$, $P_0(Y = y \mid \frob(X) = \frob(x)) = \softmax(\alpha \frob(x))_y$.
\end{definition}
\pl{I don't know how to follow the order of quantifiers, which should matter}
\ak{Made it more precise, is this better?}
\pl{I'm actually pretty confused by this definition - what is being suppressed and what's missing?}
\ak{The features to the standard model are `missing'. Suppressed is the next one.}

\textbf{Suppressed features.} In some datasets, such as satellite remote sensing datasets~\citep{jean2016combining,xie2021innout}, a standard model can latch onto country-specific features that may be less prevalent OOD.

\begin{definition}[suppressed features]
\label{dfn:suppressed_spurious}
  A distribution $P_\tau$, for $\tau \in \R^+$, is said to have suppressed features if for all $x \in \cX$ and $f \in \{\fstd, \frob\}$, $P_{\tau}(Y = y \mid f(X) = f(x)) = \softmax(\tau f(x))_y$
\end{definition}
% \ar{A bit funny to have the same suppression factor for both ID and OOD---is this necessary? Sorry, didn't have time to go over the arithmetic. Ideally, it should be written in a way that the robust features are not suppressed or suppressed less than the spurious features. I had a comment from before, but following up on that, why is it called ``suppressed spurious'' but the suppression seems symmetric wrt the features}
% \ak{Sorry, forgot to answer this. Yeah it needs to be the same, otherwise you can find a counter-example. But maybe some more general condition would work! Also, changed it to suppressed features.}

\pl{these conditions look pretty strong since they look like calibration}
\ak{Yes, it's strong, although it's not calibration because of the $\tau$ term. I'm trying to capture the fact that the model can be substantially overconfident in its prediction. What makes it strong is that $\tau$ needs to be the same for $\fstd, \frob$, but that's unfortunately what the math needs. To truly understand ensembling and why it works well, we need to understand what features deep learning learns and how they combine.}

\textbf{Anticorrelated spurious.} In some settings, the spurious feature can be correlated with a label ID but \emph{anticorrelated} OOD. For example, in Waterbirds~\citep{sagawa2020group}, the task is to classify if an image contains a waterbird or a landbird where in the ID dataset, waterbirds are primarily featured with water backgrounds and landbirds with land backgrounds, but in the OOD datasets the backgrounds are flipped such that landbirds occur with water backgrounds and vice versa. This motivates the final definition of spurious shifts where the spurious features (background) are anticorrelated with the label OOD. 
% \ar{adversarial doesn't imply anticorrelated. I am in favor or anticorrelated spurious rather than adversarial spurious because adversarial means different things to different people}
% \ak{Fixed}

\begin{definition}[anticorrelated spurious]
  A distribution $\Padv$ is said to be \adv{} if for some $\alpha, \beta > 0$, for all $x \in \cX$, $\Padv(Y = y | \fstd(x)) = \softmax(-\beta \fstd(x))_y$ (note the minus sign), while $\Padv(Y = y \mid \frob(x)) = \softmax(\alpha \frob(x))_y$.
  \end{definition}

If the OOD distribution is a mixture of suppressed features and missing spurious features, then the ensemble $\fens$ gets the best of both worlds.
\newcommand{\suppMissingEnsWorksText}{
If the OOD contains a mixture of suppressed features and missing spurious features i.e., $\Pood = \alpha P_{\tau} + (1 - \alpha) P_0$, and $P_{\tau}$ and $P_0$ are class-balanced, then we have $\Errood(\fens) \leq \Errood(\frob)$ and $\Errood(\fens) \leq \Errood(\fstd)$.
}
\begin{proposition}
\label{prop:supp_missing_ens_works}
\suppMissingEnsWorksText{}
\end{proposition}

On the other hand, if the OOD distribution contains \adv{} features, then the accuracy of $\fens$ is in between the standard and robust models.
\newcommand{\antiCorrelatedEnsFailsText}{
If spurious features are anticorrelated OOD so that $\Pood = \Padv$, then even if $\Padv$ is class-balanced, $\Errood(\frob) \leq \Errood(\fens) \leq \Errood(\fstd)$.
}
\begin{proposition}
\label{prop:anti_correlated_ens_fails}
\antiCorrelatedEnsFailsText{}
\end{proposition}
The full proofs appear in Appendix~\ref{app:analysis_appendix}.

\section{Datasets}
\label{sec:datasets}
\ak{This section needs to be compressed, maybe just one line for each dataset and intervention, and move the rest of the details to the Appendix}

We consider fourteen standard datasets, spanning multiple robustness interventions, types of shifts, and modalities (vision, language, time-series).
We first describe the robustness interventions we consider, and then describe the datasets and types of shifts.
All the datasets have been used by prior works on robustness, so we use their model checkpoints for reliable comparisons.
See Appendix~\ref{sec:more-info-datasets-appendix} for more details.
% \ak{Say all these datasets have been used by prior work and we use their checkpoints---but the purpose of this is to show the problem occurs widely}

% We run experiments on two remote sensing datasets used in prior work studying ID-OOD tradeoffs~\citep{xie2021innout}.
% These datasets consist of a core input $x$ (image data or time series data) and metadata $z$ (e.g., location, meteorological climate data). 
% The metadata is spuriously correlated with the target---using the metadata to predict labels improves accuracy in-distribution (ID), but hurts accuracy out-of-distribution.
% ~\citet{xie2021innout} consider a standard model that takes in both the core inputs and metadata to predict the target, and a robust model that only takes in the core inputs and does some additional pretraining.
% They call these the `aux-in' and `aux-out' models respectively.

\textbf{Robustness interventions}:
\begin{enumerate}
	\item In-N-Out:~\citet{xie2021innout} use domain knowledge to project out spurious features in the input, and do an additional pretraining step. They call this robust model ``aux-out'' and show that it improves accuracy OOD, but hurts accuracy ID, compared to ERM (which they call ``aux-in'').
	% We use their datasets and model checkpoints.
% They call these the `aux-in' and `aux-out' models respectively.
	\item Lightweight fine-tuning:
	% Recent works show that tuning only parts of a pretrained model can often do better OOD even though the ID performance is worse~\citep{li2021prefix,houlsby2019parameter}.
	We take checkpoints from~\citet{kumar2022finetuning} where the standard model fine-tunes all parameters of a pretrained model on an ID dataset, and the robust model only learns the top linear `head' layer (which does better OOD but worse ID).
	\item Zero-shot language prompting: CLIP~\citep{radford2021clip} is a multi-modal model that can predict the label of an image by comparing the image embedding, with prompts such as `photo of an apple'. They show that this zero-shot language prompting approach (robust model) is more accurate OOD than fine-tuning the model (standard model), although ID accuracy of the robust model is worse.
	\item Group distributionally robust optimization (DRO)~\citep{sagawa2020group}: Standard ERM models often latch on to spurious correlations in a dataset, such as image background color, or the occurrence of certain words in a sentence. Group DRO essentially upweights examples where this spurious correlation is not present.
	\item CORAL~\citep{sun2016deep} aims to align feature representations across different domains, by penalizing differences in the means and covariances
	of the feature distributions.
	The hope is that this generalizes better to OOD domains.
	% The original formulation in~\citet{sagawa2020group} assumes the spurious correlations are annotated, but newer variants~\citep{liu2021jtt} can work even without these annotations.
 \end{enumerate}

We consider three types of \natshifts{} (geography shifts, subpopulation shifts, style shifts), and we also consider adversarially synthesized ``anticorrelated'' spurious shifts.

\textbf{Geography shifts.} In geography shifts the ID data comes from some locations, and the OOD data comes from a different set of locations. One motivation is that in many developing areas training data may be unavailable because of monetary constraints~\citep{jean2016combining}.
\begin{enumerate}
	\item \textbf{LandCover}~\citep{russwurm2020meta}: The goal is to classify a satellite time-series measured by NASA's MODIS satellite~\citep{modis2015landcover} into one of 6 land types (e.g., "grassland", "savannas"). The ID data contains time-series from outside Africa, and the OOD data consists of time-series from Africa.~\citet{xie2021innout} use the In-N-Out intervention.
	\item \textbf{Cropland}~\citep{wang2020weakly}: The goal is to predict whether a satellite image is of a cropland or not. The ID dataset contains images from Iowa, Missouri, and Illinois, and the OOD dataset contains images from Indiana and Kentucky.~\citet{xie2021innout} use the In-N-Out intervention.
	\item \textbf{iWildCam}~\citep{beery2020iwildcam,koh2021wilds}: The goal is to classify the species of an animal given a photo taken by a camera placed in the wild. The ID dataset consists of photos taken by over 200 cameras, and the OOD dataset consists of photos taken by held-out cameras placed in different locations.~\citet{koh2021wilds} use the CORAL intervention.
\end{enumerate}
\ak{Optionally move the discussion of what's core and spurious to the Appendix}

\textbf{Subpopulation shifts.} In subpopulation shifts, the ID data contains a few sub-categories (e.g., black bear and sloth bear), and the OOD data contains different sub-categories (e.g., brown bears and polar bears) of the same parent category (e.g., bears). For both datasets below,~\citet{kumar2022finetuning} use the lightweight fine-tuning intervention.
\begin{enumerate}
	\item \textbf{Living-17}~\citep{santurkar2020breeds}: the goal is to classify an image as one of 17 animal categories such as ``bear'', where the ID and OOD datasets have different species of bears. 
	\item \textbf{Entiy-30}~\citep{santurkar2020breeds}: similar to Living-17, except the goal is to classify an image as one of 30 entity categories such as ``food'', ``motor vehicle'', and ``insect''.
\end{enumerate}
\ak{Optionally move the discussion of what checkpoints we used to the Appendix.}

\textbf{Style shifts.} In style shifts, the ID data has a certain style (e.g., sketches), and the OOD data has a different style (e.g., real photos, renditions). 
\begin{enumerate}
	\item \textbf{DomainNet}~\citep{peng2019moment}: a standard domain adaptation dataset. Here, our ID dataset contains ``sketch'' images (e.g., drawings of apples, elephants, etc), and the OOD dataset contains ``real'' photos of the same categories.~\citet{kumar2022finetuning} use the lightweight fine-tuning intervention.
	\item \textbf{CelebA}~\citep{liu2015deep}: the goal is to classify a portrait of a face as ``male'' or ``female'' - the ID dataset contains images of people without hats, and the OOD dataset contains images of people wearing hats (some facial features might be ``suppressed'' or ``missing'' with hats).~\citet{xie2021innout} use the In-N-Out intervention.
	\item \textbf{CIFAR->STL}: standard domain adaptation dataset~\citep{french2018selfensembling}, where the ID is CIFAR-10~\citep{krizhevsky2009learningmultiple}, and the OOD is STL~\citep{coates2011stl10}. The task is to classify an image into one of 10 categories such as ``dog'', ``cat'', or ``airplane''.~\citet{kumar2022finetuning} use the lightweight fine-tuning intervention.
	\item \textbf{ImageNet}~\citep{russakovsky2015imagenet}: a large scale dataset where the goal is to classify an image into one of 1000 categories.~\citet{radford2021clip} use the zero-shot language prompting intervention. We evaluate on 3 standard OOD datasets: \textbf{ImageNetV2}~\citep{recht2019doimagenet},\textbf{ImageNet-R}~\citep{hendrycks2020many}, and \textbf{ImageNet-Sketch}~\citep{wang2019learningrobust}.
\end{enumerate}

\textbf{\Advshifts{}.} In these adversarially synthesized shifts, the ID dataset contains a feature that is correlated with a label, but this correlation is flipped OOD.~\citet{jones2021selective} use the group DRO intervention. Note that \emph{for the OOD numbers, we measure the standard worst group accuracy (across spurious and label annotations), as done by~\citet{jones2021selective}}.
% For example, waterbirds is explicitly constructed so that ``water'' backgrounds are correlated with ``waterbird'' labels in the ID, but anti-correlated OOD.
\begin{enumerate}
	\item \textbf{Waterbirds}~\citep{sagawa2020group}: The goal is to classify an image as a ``waterbird'' or ``landbird''. The dataset is synthetically constructed to have \adv{} features: ``water'' backgrounds are correlated with ``waterbird'' labels in the ID, but anticorrelated OOD.
	\item \textbf{MNLI}~\citep{williams2018broad}: The goal is to predict whether a hypothesis is entailed, contradicted by, or neutral to an associated premise.~\citet{sagawa2020group} partition the dataset so that ``negation'' words are correlated with the contradiction label ID but these words are anticorrelated with the contradiction label OOD.
	\item \textbf{CivilComments}~\citep{borkan2019nuanced}: The goal is to predict whether a comment is toxic or not.~\citet{jones2021selective} partition the dataset so that in the ID split mentions of a Christian identity are correlated with non-toxic comments, but in the OOD split mentions of a Christian identity are correlated with a toxic comment. CivilComments is also used in~\citet{koh2021wilds}.
\end{enumerate}

\section{Results}
\label{sec:experiments}

\begin{table*}[t]
\caption{
~\citet{xie2021innout} propose In-N-Out (self-training) to mitigate ID-OOD accuracy tradeoffs---their method requires lots of unlabeled data.
Even without this unlabeled data, \calens{} are competitive with or outperform self-training ID and OOD.
We show results on all datasets used by~\citet{xie2021innout}.
}
\label{tab:self_train_results}
\vskip 0.15in
\begin{center}
\begin{tabular}{ccccccc}
\toprule
                      & \multicolumn{2}{c}{Cropland}            & \multicolumn{2}{c}{Landcover} & \multicolumn{2}{c}{CelebA} \\
                      & ID Acc                  & OOD Acc                 & ID Acc                  & OOD Acc   & ID Acc                  & OOD Acc   \\
\midrule
Standard model        & 95.3 (0.0)          & \textbf{85.6 (5.8)}  & 76.9 (0.3)          & 55.7 (1.1)         &  90.4 (0.5)  &  74.5 (0.6)  \\
Robust model          & 95.1 (0.1)          & 89.8 (0.4)             & 72.7 (0.2)           & \textbf{60.4 (1.1)}      &   \textbf{94.5 (0.2)}  &  76.3 (1.2) \\
Self-training         & 95.3 (0.2)          & \textbf{90.6 (0.6)} & \textbf{77.0 (0.4)} & \textbf{61.0 (0.7)}     & 93.1 (0.2) & \textbf{78.7 (0.7)} \\
Cal ensembling & \textbf{95.6 (0.1)} & \textbf{91.3 (0.8)} & \textbf{77.2 (0.2)} & \textbf{60.8 (0.8)}     &  \textbf{94.5 (0.5)} & \textbf{77.6 (1.2)} \\
\bottomrule
\end{tabular}
\end{center}
\vskip -0.1in
\end{table*}

% In the theoretical section, we showed that ensembling can work very well (gets the best of the standard and robust models) when there are spurious correlations ID that are suppressed or missing OOD
% But if the spurious correlations are more adversarial (anti-correlated) OOD, the ensembling does not do so well.
% We run experiments on a wide array of benchmark datasets, spanning blah
% Connecting geography, style, etc, shifts, to suppressed
% We run experiments on a wide range of naturally occuring shifts (geography shifts, style shifts, subpopulation shifts)
% We also run 
% \ak{Should we recap the goal again: goal is to get strong ID accuracy of standard model, robust accuracy of OOD model}
% \ak{Maybe say self-training is on the 3 datasets used by prior work}
% Our analysis (Section~\ref{}) predicts that calibrated ensembles can get the best of both worlds for \natshifts{}, but not for \advshifts{}.
In Section~\ref{sec:experiments-can-mitigate}, we show that \calens{} get the best of both worlds across the \numnat{} \natshifts{} we consider, but not on the \numadv{} adversarially synthesized \advshifts{}, as predicted by our analysis in Section~\ref{sec:analysis}.
\Calens{} match or outperform a prior state-of-the art approach based on self-training~\citep{xie2021innout}, which requires additional unlabeled data.
In Section~\ref{sec:experiments-how-ensemble}, we show ablations of our method.
Interestingly, we find that a common approach of tuning the ensemble weights to optimize ID accuracy gets lower OOD accuracies than \calens{}.
% we show that \emph{how} we ensemble the models is key: tuning the ensemble weights to optimize ID accuracy leads to poor OOD performance.
% However, calibrating (also only on \emph{in-distribution} data) leads to improved accuracies, as suggested by our analysis in Section~\ref{sec:analysis_id}.
% We sanity check that ensembling two standard or two robust models does not work as well (even with calibration).
% In Section~\ref{sec:experiments_models_miscalibrated}, we show that even after calibrating on ID data, the standard and robust models are miscalibrated OOD---but ensembles are still able to effectively combine their predictions.

\begin{table*}[t]
\caption{
\emph{
In-distribution (ID)} accuracies for the standard model, robust model, and \calens{}, across \numidnat{} \natshift{} datasets (colored blue) and \numadv{} \advshift{} datasets (colored red and starred).
On the \numidnat{} ID \natshift{} datasets, \calens{} match or outperform the best model in 8/9 cases, and on average outperforms both the standard and robust models.
For the remaining dataset, CIFAR-10, \calens{} close 97\% of the gap between the standard and robust model.
% As expected from our analysis (Section~\ref{sec:analysis}), \calens{} do not perform as well on \advshift{} datasets.
}
\label{tab:id_results}
\vskip 0.15in
\begin{center}

\begin{tabular}{cccccccc}
\toprule
& \color{blue}{Ent30} & \color{blue}{DomNet} & \color{blue}{CIFAR10} & \color{blue}{Liv17} & \color{blue}{Land} & \color{blue}{Crop} & \color{blue}{CelebA}\\
\midrule
Standard & \textbf{93.6 (0.2)} & 83.9 (1.0) & \textbf{97.4 (0.1)} & 96.9 (0.1) & \textbf{76.9 (0.3)} & 95.3 (0.0) & 90.4 (0.5)\\
Robust & 90.7 (0.2) & 89.2 (0.1) & 92.0 (0.0) & 97.0 (0.0) & 72.7 (0.2) & 95.1 (0.1) & \textbf{94.5 (0.2)}\\
Cal Ensemble & \textbf{93.7 (0.1)} & \textbf{91.2 (0.7)} & 97.2 (0.1) & \textbf{97.2 (0.2)} & \textbf{77.2 (0.2)} & \textbf{95.6 (0.1)} & \textbf{94.5 (0.5)}\\
\bottomrule
\end{tabular}
\vspace{1.2mm}
\newline
\begin{tabular}{ccc|ccc}
\toprule
 & \color{blue}{ImageNet} & \color{blue}{iWildCam} \;\;\; & \;\;\; \color{red}{MNLI*} & \color{red}{Waterbirds*} & \color{red}{CivilComments*}\\
\midrule
Standard & 81.7 (-) & 82.4 (-) \;\;\; & \;\;\; \textbf{82.9 (-)} & 88.3 (-) & \textbf{92.8 (-)}\\
Robust & 68.4 (-) & 81.8 (-) \;\;\; & \;\;\; 81.5 (-) & \textbf{93.2 (-)} & 86.3 (-)\\
Cal Ensemble & \textbf{82.0 (-)} & \textbf{84.0 (-)} \;\;\; & \;\;\; \textbf{82.8 (-)} & 92.9 (-) & 91.4 (-)\\
\bottomrule
\end{tabular}

\end{center}
\vskip -0.1in
\end{table*}

\begin{table*}[t]
\caption{
\emph{Out-of-distribution (OOD)} accuracies for the standard model, robust model, and \calens{}, across \numnat{} \natshift{} datasets (colored blue) and \numadv{} \advshift{} datasets (colored red and starred).
On the \numnat{} OOD \natshift{} datasets, \calens{} match or outperform the best model in 10/11 cases, and on average outperforms both the standard and robust models.
For the remaining dataset, DomainNet, \calens{} close 96\% of the gap between the standard and robust model.
As expected from our analysis (Section~\ref{sec:analysis}), on \advshifts{} the accuracy of \calens{} is between the standard and robust models.
}
\label{tab:ood_results}
\vskip 0.15in
\begin{center}

\begin{tabular}{cccccccc}
\toprule
& \color{blue}{Ent30} & \color{blue}{DomNet} & \color{blue}{STL10} & \color{blue}{Liv17} & \color{blue}{Land} & \color{blue}{Crop} & \color{blue}{CelebA}\\
\midrule
Standard & 60.7 (0.1) & 55.3 (0.4) & 82.4 (0.3) & 77.7 (0.6) & 55.7 (1.1) & \textbf{85.6 (5.8)} & 74.5 (0.6)\\
Robust & 63.2 (1.1) & \textbf{87.2 (0.1)} & 85.1 (0.2) & \textbf{82.2 (0.2)} & \textbf{60.4 (1.1)} & 89.8 (0.4) & 76.3 (1.2)\\
Cal Ensemble & \textbf{64.7 (0.5)} & 86.1 (0.2) & \textbf{87.3 (0.2)} & \textbf{82.2 (0.6)} & \textbf{60.8 (0.8)} & \textbf{91.3 (0.8)} & \textbf{77.6 (1.2)}\\
\bottomrule
\end{tabular}
\vspace{1.2mm}
\newline
\begin{tabular}{ccccc|ccc}
\toprule
 & \color{blue}{ImNet-R} & \color{blue}{ImNet-V2} & \color{blue}{ImNet-Sk} & \color{blue}{iWildCam} \;\;\; & \;\;\; \color{red}{MNLI*} & \color{red}{Waterbirds*} & \color{red}{Comments*}\\
\midrule
Standard & 52.4 (-) & 71.5 (-) & 40.5 (-) & 61.1 (-) \;\;\; & \;\;\; 65.5 (-) & 60.4 (-) & 56.8 (-)\\
Robust & 77.5 (-) & 61.9 (-) & 48.2 (-) & 63.0 (-) \;\;\; & \;\;\; \textbf{77.4 (-)} & \textbf{88.1 (-)} & \textbf{84.2 (-)}\\
Cal Ensemble & \textbf{77.9 (-)} & \textbf{73.2 (-)} & \textbf{52.3 (-)} & \textbf{66.3 (-)} \;\;\; & \;\;\; 73.2 (-) & 81.1 (-) & 71.8 (-)\\
\bottomrule
\end{tabular}

\end{center}
\vskip -0.1in
\end{table*}

\subsection{Main Results}
\label{sec:experiments-can-mitigate}

\textbf{Competitive with self-training.}
~\citet{xie2021innout} propose self-training on unlabeled data to mitigate ID-OOD accuracy tradeoffs.
We run experiments on all 3 datasets they consider (Landcover, Cropland, CelebA), taking checkpoints from the official CodaLab implementation of~\citet{xie2021innout}.
% \footnote{We compare to the In-N-Out results on their official CodaLab worksheet. Self-training requires additional unlabeled data which is not available in many datasets.}
% \footnote{We compare to the In-N-Out results on their official CodaLab worksheet.}
Table~\ref{tab:self_train_results} shows that \calens{} match or outperform self-training on all 3 of their datasets, both ID and OOD.
We believe this is interesting because our method is simple and does not need additional unlabeled data (which, for example, the other datasets do not have).
Note that~\citet{xie2021innout} also consider additional self-training on OOD (+ID) unlabeled data, which we do not compare with because it uses additional data (the OOD data).
\pl{this is surprising; presumably we'd do better if we used both unlabeled data and calibrated ensembles?}
\ak{Yeah, that sounds reasonable. I added a discussion of this to future work, and responded on Slack. Self-training has three benefits: 1. May potentially work on anticorrelated spurious features (based on some intuitive/handwavy examples I have), 2. You get one single model at the end, so faster inference time. Although distillation could get you the same benefits. 3. In-N-Out plus additional self-training on \emph{OOD unlabeled data} outperforms calibrated ensembles.}

\ak{maybe reference the green and red/starred styling of the tables in the text}
\textbf{Strong ID and OOD accuracy.}
Across the \numnat{} \natshifts{}, \calens{} get the best of both worlds, typically outperforming the standard and robust model both ID (Table~\ref{tab:id_results}) and OOD (Table~\ref{tab:ood_results}).
Averaged across the \emph{\natshift{}} datasets, \calens{} get \calaccidnatural\% ID (vs. \stdaccidnatural{}\% for the standard model and \robaccidnatural{}\% for the robust model) and \calaccoodnatural{}\% OOD (vs. \robaccoodnatural{}\% for the robust model and \stdaccoodnatural{}\% for the standard model).
The method works across the board---\calens{} achieve the best performance on 8/9 ID \natshifts{}, and on 10/11 OOD \natshifts{}.
For the remaining two cases, DomainNet OOD and CIFAR-10 ID, \calens{} close over 95\% of the gap between the standard and robust model.
% \pl{for ID or OOD?}\ak{mentioned it before the comma}
 % (96\% for DomainNet, 97\% for CIFAR-10).

\textbf{Shift type is important.}
Our analysis in Section~\ref{sec:analysis} predicts that \calens{} \emph{do not} work as well on \advshifts{}, where a spurious feature is correlated with the label but anticorrelated OOD.
Indeed, in these cases the OOD accuracy of \calens{} is between the standard and robust model (Table~\ref{tab:ood_results}).
Even so, averaged across all \numtotal{} datasets \calens{} do well and get \calaccid{}\% ID (vs. \stdaccid{}\% for the standard model, \robaccid{}\% for the robust model) and \calaccood{}\% OOD (vs. \stdaccood{}\% for the standard model, \robaccood{}\% for the robust model).

% In addition, unlike self-training, \calens{} handle both models symmetrically and does not require prior knowledge of which model is better in what domain.
% \ak{There is a slight gap with the theory---we predicted ensembles }

% \textbf{Strong ID and OOD accuracy}:
% Calibrating and then ensembling a standard and a robust model, gets the best of both worlds, typically outperforming the standard and robust model both ID (Table~\ref{tab:id_results}) and OOD (Table~\ref{tab:ood_results}).
% Averaged across the datasets, \calens{} get 89.3\% ID (vs 87.9\% for the standard model and 84.7\% for the robust model) and 77.9\% OOD (vs 77.2\% for the robust model and 64.9\% for the standard model).
% The method works across the board---\calens{} achieve the best performance on 5/6 ID datasets, and on 5/6 OOD datasets.
% For the remaining two datasets, DomainNet OOD and CIFAR-10 ID, \calens{} close over 95\% of the gap between the standard and robust model (96\% for DomainNet, 97\% for CIFAR-10). 

\subsection{Ablations}
\label{sec:experiments-how-ensemble}

Our proposed method is a simple combination of a calibrated robust and calibrated standard model.
We vary the components of our method and try: (i) tuned ensembles without calibration, (ii) vanilla ensembles without calibration, and (iii) ensembles of two standard or two robust models.

\textbf{Tuned ensembles do not mitigate tradeoffs.} A natural way to ensemble the two models is ``tuned ensembles'': choosing the ensemble weights to optimize accuracy on the ID validation set. This approach is also known as stacking, and has performed well on the Netflix prize and Kaggle competitions~\citep{sill2009feature}.
% ~\citet{miller2021line} report that OOD accuracy can be quite correlated with ID accuracy, which suggests that this method might do well OOD.\ak{this line might be confusing}
Interestingly, we find that tuned ensembles do not do very well OOD, getting an average accuracy of \tunedaccood{}\% across the \numtotal{} datasets (vs. \calaccood{}\% for \calens{}).
The ID accuracies are similar---results for all datasets are in Table~\ref{tab:id_tuned} (ID) and Table~\ref{tab:ood_tuned} (OOD).
% Tuned ensembles do marginally better ID getting an average accuracy of \tunedaccid{}\% (vs. \calaccid{}\% for \calens{}). Naturally, we expect the tuned ensemble to do the best ID since its weights are tailored for ID---what is surprising is that the \calens{} do so much better OOD without using any OOD data either. We show results for all datasets in Table~\ref{tab:id_tuned} (ID) and Table~\ref{tab:ood_tuned} (OOD). 

\textbf{Calibration helps.} \calens{} (calibration is only done on ID data) outperform vanilla ensembles, especially on OOD test examples. \calens{} get an average ID accuracy of \calaccid{}\% (vs. \naiveaccid{}\% \pl{is this statistically signicant?}\ak{For OOD, yes. The longer answer is that per-dataset results are in Table 5. As usual, we include confidence intervals for the 7 datasets where we have multiple runs (e.g., didn't run multiply runs on ImageNet). It is statistically significantly better on 3 of the 7 and the same on 4. On the reamining 7 where we don't have confidence intervals, it does much better on 4, about the same on 2, and very slightly worse on 1.} for vanilla ensembles) and an average OOD accuracy of \calaccood{}\% (vs. \naiveaccood{}\% for vanilla ensembles). We show results for all datasets in Table~\ref{tab:id_tuned} (ID) and Table~\ref{tab:ood_tuned} (OOD).

\textbf{Outperforms standard and robust ensembles.}
As a sanity check, Appendix~\ref{sec:per-dataset-ensemble-ablations} shows that our method outperforms 1. ensembling two (calibrated) standard models, and 2. ensembling two (calibrated) robust models.
% For these ensembles, calibration does not affect the accuracy much.

\textbf{Models are miscalibrated OOD.}
Even after ID calibration, we find that the standard and robust models are not calibrated OOD, which matches prior work~\citep{ovadia2019uncertainty}.
We estimate the expected calibration error (ECE; Equation 2 in~\citet{guo2017calibration}).
Since we calibrated on ID data, the ECE is low ID (\stdeceid\% for the standard model, \robeceid\% for the robust model; Table~\ref{tab:id_ece}).
However, the ECE is high OOD (\stdeceood\% for the standard model, \robeceood\% for the robust model; Table~\ref{tab:ood_ece})
% ---this agrees with prior work~\citep{ovadia2019uncertainty} shows that models calibrated ID are still not calibrated OOD.
Appendix~\ref{sec:per-dataset-calibration-appendix} shows that even the relative confidence of the models can be wrong: the standard model can be \emph{more confident} but \emph{less accurate} OOD, after ID-calibration.
Nonetheless, \calens{} get the best of both worlds---see Section~\ref{sec:analysis} for some simple intuitions for why this can happen.

\section{Related Works and Discussion}

\paragraph{Calibration.} Calibration has been widely studied in machine learning~\citep{naeini2014binary, guo2017calibration, kumar2019calibration}, and applications such as meteorology~\citep{murphy1973vector,degroot1983forecasters,gneiting2005weather}, fairness~\citep{johnson2018multicalibration}, and healthcare~\citep{jiang2012calibrating}. Many of these works focus on the in-distribution (ID) setting, where models are calibrated on the same distribution that they are evaluated on.~\citet{ovadia2019uncertainty,jones2021selective} show that if we calibrate (e.g., via temperature scaling) a model ID, it still has poor uncertainties OOD.
% ~\citet{} also show that model uncertainties can be quite unreliable out-of-distribution.
However, we show that despite having poor uncertainties on traditional metrics, calibrated models can be combined effectively to mitigate ID-OOD tradeoffs.
~\citet{wald2021calibration} show that if a model is calibrated on many domains (domains $>$ no. of features) in a linear setting, then the model is calibrated (and invariant) on new domains. A key difference is that they require a large number of training domains, which may need to be annotated to ensure calibration across them, while we only require access to a single domain.

\paragraph{Ensembling.}
Ensembling models is a common way to get an accuracy boost---typically the ensemble members are trained with a different random seed~\citep{lakshminarayanan2017simple} or augmentation~\citep{stickland2020diverse}.
In the setting where the ensemble members mostly differ by random seeds or augmentations, prior work has shown that calibrating the members of an ensemble does not help~\citep{wu2021ensemble,ovadia2019uncertainty}.
% Indeed, we find that calibration has minimal effect when we ensemble two standard, or two robust models, that are trained from different seeds.
However when we combine two very different models (standard and robust), calibration leads to clear improvements.

\paragraph{Mitigating ID-OOD tradeoffs.} Tradeoffs between ID and OOD accuracy are widely studied and prior work self-trains on large amounts of unlabeled data to mitigate such tradeoffs~\citep{raghunathan2020understanding, xie2021innout, khani2021removing}.
In contrast, our approach uses no extra unlabeled data and is a simple method where we just add up the model probabilities after a quick calibration step.
In concurrent and independent work,~\citep{wortsman2021robust} show that there \emph{exists} a way to combine a CLIP zero-shot and fine-tuned model to get good ID and OOD accuracy---however learning how to combine the models may require OOD data, which is not available. We show that the natural way to learn how to weight ensemble members---selecting the weights to optimize ID accuracy---does not get the best of both worlds.
In addition, their approach does not directly apply to settings where the standard and robust models have different architectures, such as In-N-Out~\citep{xie2021innout}.

\paragraph{Conclusion and Future Work.} 
In this paper, we show that \calens{}, a simple method of calibrating a standard and robust model only on ID data and then ensembling them, can eliminate the tradeoff between in-distribution (ID) and out-of-distribution (OOD) accuracy on a wide range of natural shifts.
We hope that this leads to more widespread use and deployment of robustness interventions.

\calens{} were competitive with prior work that used self-training, despite being simpler and not using additional unlabeled data.
However, self-training may have advantages: we believe self-training may potentially eliminate tradeoffs even in anticorrelated spurious settings---it could be interesting for future work to compare ensembling and self-training theoretically, and see if their benefits are complementary.
Additionally,~\calens{} require twice the compute of a single model (although for fairness, we compared with an ensemble of standard or robust models), while self-training gives us a single model.
One potential future direction is to see if~\calens{} can be distilled into a single model (without additional unlabeled data). Self-training could also be complementary to \calens{}, and give even better accuracies when combined.

\section{Acknowledgements}

This work was in part supported by the Open Philanthropy Project and NSF Award Grant No. 1805310, and NSF IIS 2045685.
AR was supported by an Open Philanthropy Project AI Fellowship.
We would like to thank Robbie Jones and the anonymous reviewers for helpful comments on our draft.

\bibliography{all.bib}
\appendix

\onecolumn
\section{Proofs for Section~\ref{sec:analysis}}
\label{app:analysis_appendix}

We begin with some standard background on Bayes optimal classifiers.
When then prove the results in Section~\ref{sec:analysis}.
By default, expectations are taken over all random variables.

\subsection{Background on Bayes-optimal classifiers}

These results are all standard, but we include it as background information since different texts use different notations.
Let $Z \in \cZ$ denotes some features (that can be complicated functions of the input $x$, for example the output of a neural network), and let $Y \in \cY$ denote the label.
Let $P$ be a distribution over $(Z, Y)$.
The Bayes-optimal classifier predicts the most likely label $y$ given features $z$.

\begin{definition}
\label{dfn:bayes-opt-appendix}
The Bayes-optimal classifier for $P$ given features $z$ is given by:
\begin{equation}
	y_*(z) = \argmin_{y \in \cY} P(y \mid z).
\end{equation}
\end{definition}

The Bayes-optimal classifier has the minimum misclassification error of all possible classifiers that use $z \in \cZ$ to predict $y \in \cY$.
Formally, the error of a classifier $\hat{y}$ is the probability that it gets the label incorrect.

\begin{definition}
\label{dfn:error-appendix}
The error of a predictor $\hat{y} : \cZ \to \cY$ on distribution $P$ is given by:
\begin{equation}
	\Err_P(\hat{y}) = P(Y \neq \hat{y}(Z)),
\end{equation}
\end{definition}

Alternatively, we can look at the error for each $Z$, and then take the average over $Z$, which gives us:

\begin{lemma}
\label{lem:alt_error_appendix}
The error of a predictor $\hat{y} : \cZ \to \cY$ on distribution $P$ can be written as:
\begin{equation}
	\Err_P(\hat{y}) = \E[1 - P(Y=\hat{y}(Z) \mid Z)].
\end{equation}
\end{lemma}

\begin{proof}
	We can write the misclassification probability as an expectation over an indicator and then apply the law of total expectation.
	\begin{align}
		P(Y \neq \hat{y}(Z)) &= \E[\mathbb{I}(Y \neq \hat{y}(Z))] \\
		&= \E[\E[\mathbb{I}(Y \neq \hat{y}(Z)) \mid Z]].
	\end{align}
	And then just write the inner expectation as a probability.
	\begin{align}
		\E[\E[\mathbb{I}(Y \neq \hat{y}(Z)) \mid Z]] &= \E[P(Y \neq \hat{y}(Z) \mid Z)] \\
		&= \E[1 - P(Y = \hat{y}(Z) \mid Z)].
	\end{align}
\end{proof}

The Bayes-optimal classifier selects the $y$ with the highest probability given $z$, so we have:

\begin{lemma}
\label{lem:bayes-opt-err-appendix}
The error of the Bayes-optimal classifier $y_*$ on a distributon $P$ can be written as (where $Z \sim P$):
\begin{equation}
	\Err_P(y_*) = \E[1 - \max_{y \in \cY} P(Y=y \mid Z)].
\end{equation}
\end{lemma}

\begin{proof}
	The proof is immediate by substituting the definition of the Bayes-optimal classifier (Definition~\ref{dfn:bayes-opt-appendix}) into the alternative formula for the error in Lemma~\ref{lem:alt_error_appendix}.
\end{proof}

From the above, it is clear that the Bayes-optimal classifier has lower error than any other classifier that uses only $z$, formalized below.

\begin{lemma}
\label{lem:bayes-opt-is-optimal-appendix}
The bayes-optimal classifier (for $P$) has lower error than all classifiers $\hat{y} : \cZ \to \cY$:
\begin{equation}
	\Err_P(y_*) \leq \Err_P(\hat{y}).
\end{equation}
\end{lemma}

\begin{proof}
Beginning from Lemma~\ref{lem:alt_error_appendix}, we have:
\begin{align}
	\Err_P(\hat{y}) &= \E[1 - P(Y=\hat{y}(Z) \mid Z)] \\
	&\geq \E[1 - \max_{y \in \cY} P(Y=y \mid Z)] \\
	&= \Err_P(y_*).
\end{align}
\end{proof}

As a simple corollary, we note that the accuracy of the Bayes-optimal classifier is at least the frequency of the most common label. 

\begin{corollary}
\label{cor:bayes-opt-better-trivial}
If $y_*$ is bayes-optimal for $P$ then,
	\begin{equation}
		\Err_P(y_*) \leq 1 - \max_{y \in \cY} P(Y=y)
	\end{equation}
\end{corollary}

So for example if $P$ is balanced, then the Bayes-opt classifier will have accuracy at least $1/K$, where $K$ is the number of classes.

Note that calibrated classifiers are Bayes-optimal given their outputs.
Formally, let $P$ be a distribution over $(x, y)$, and suppose $f$ is calibrated with respect to $P$.
Let $z = f(x)$ and let $P'$ be the induced distribution over $(z, y)$.
Then $f$ is Bayes-optimal for $P'$ given features $z$.
The label distributions $P'(y)$ and $P(y)$ are the same, so Lemma~\ref{lem:bayes-opt-is-optimal-appendix} applies to any calibrated classifier.

\subsection{Proof of Proposition~\ref{prop:calibration-ensemble-optimal}}

\newtheorem*{calibrationEnsembleOptimalProp}{Restatement of Proposition~\ref{prop:calibration-ensemble-optimal}}

\begin{calibrationEnsembleOptimalProp}
\calibrationEnsembleOptimalText
\end{calibrationEnsembleOptimalProp}

We first show that in the setting of the Proposition, we can write $P(y \mid \frob(x), \frob(x))$ in terms of $\frob(x)$ and $\fstd(x)$.

\begin{lemma}
	\label{lem:bayes_prob_softmax}
	In the setting of Proposition~\ref{prop:calibration-ensemble-optimal}, let $m \in \R^K$ be the log of the marginal probabilities $P(y)$:
	\begin{equation}
		\label{eqn:general_combination_lem_bayes_prob_softmax}
		m_y = \log{P(y)}, \quad \mbox{for all }y \in [K].
	\end{equation}
	Then we have:
	\begin{equation}
		\label{eqn:bayes_prob_softmax_appendix_imbalanced}
		P(y \mid \fstd(x), \frob(x)) = \softmax(\fstd(x) + \frob(x) - m)_y, \quad \mbox{for all }y \in [K].
	\end{equation}
	In the balanced setting, where $P(y) = 1/K$ for all $y$, this simplifies to:
	\begin{equation}
		\label{eqn:bayes_prob_softmax_appendix_balanced}
		P(y \mid \fstd(x), \frob(x)) = \softmax(\fstd(x) + \frob(x))_y, \quad \mbox{for all }y \in [K].
	\end{equation}
\end{lemma}

\begin{proof}
	Fix $r = \frob(x)$ and $s = \fstd(x)$, where $r, s \in \R^K$.
	We first rewrite the probability of $y$ given the robust and standard model outputs $P(y \mid r, s)$ in terms of the probability of $y$ given each of the individual model outputs: $P(y \mid r)$ and $P(y \mid s)$.
	We do this for discrete random variables for simplicity, but the same result follows by using Bayes rule for general random variables.
	\begin{align}
		P(y \mid r, s) &= \frac{P(r, s \mid y) P(y)}{P(r, s)} && \text{[Bayes rule]} \\
		&= \frac{P(r \mid y)P(s \mid y) P(y)}{P(r, s)} && \text{[$r \perp s \mid y$]} \\
		&= \frac{[\frac{P(y \mid r) P(r)}{P(y)} \frac{P(y \mid s) P(s)}{P(y)} P(y)}{P(r, s)} && \text{[Bayes rule]} \\
		&= \frac{P(y \mid r) P(y \mid s)}{P(y)} \Big[ \frac{P(r) P(s)}{P(r, s)} \Big] && \text{[Algebra]} \\
	\end{align}
	Since $r, s$ are fixed, we can denote the terms that do not depend on $y$ by a constant $c_1$,
	\begin{equation}
		c_1 = \frac{P(r) P(s)}{P(r, s)}.
	\end{equation}
	So then we can write:
	\begin{equation}
		\label{eqn:simplified_bayes_prob_0_appendix}
		P(y \mid r, s) = \frac{P(y \mid r) P(y \mid s)}{P(y)} c_1, \quad \mbox{for all $y \in [K]$}.
	\end{equation}
	Now, we assumed $P(Y=y \mid r) = \softmax(r)_y$ and $P(Y=y \mid s) = \softmax(s)_y$ for all $y \in [K]$. 
	For some constants $c_2, c_3 \in \R$, we can write this as: $P(Y=y \mid r) = \exp(r_y) / c_2$ and $P(Y=y \mid s) = \exp(s_y) / c_3$ for all $y \in [K]$.
	Substituting this into Equation~\ref{eqn:simplified_bayes_prob_0_appendix}, we get:
	\begin{equation}
		P(y \mid r, s) = \frac{\exp(r_y + s_y)}{P(y)} \frac{c_1}{c_2 c_3}, \quad \mbox{for all $y \in [K]$}.
	\end{equation}
	Writing $1/P(y)$ as $\exp(-\log{P(y)})$, and setting $c_4 = \frac{c_1}{c_2 c_3}$, this gives us:
	\begin{equation}
		P(y \mid r, s) = c_4 \exp(r_y + s_y - \log{P(y)}), \quad \mbox{for all $y \in [K]$}.
	\end{equation}
	Since the LHS is a probability, these must sum to $1$ and so $c_4$ must be a normalizing constant, that is, $c_4 = 1 / (\sum_{y \in [K]} \exp(r_y + s_y - \log{P(y)}))$.
	This gives us:
	 \begin{equation}
		P(y \mid r, s) = \softmax(r + s - m)_y, \quad \mbox{for all $y \in [K]$},
	\end{equation}
	which is precisely Equation~\ref{eqn:bayes_prob_softmax_appendix_imbalanced}.
	In the balanced setting, we have $P(Y) = 1/K$ so we simply fold $P(Y)$ into the constant $c_4$, and get:
	\begin{equation}
		P(y \mid r, s) = \softmax(r + s)_y, \quad \mbox{for all $y \in [K]$},
	\end{equation}
	which is precisely Equation~\ref{eqn:bayes_prob_softmax_appendix_balanced}.
\end{proof}

Now we are ready to prove Proposition~\ref{prop:calibration-ensemble-optimal}.

\begin{proof}[Proof of Proposition~\ref{prop:calibration-ensemble-optimal}]
	We assumed the ``balanced'' setting where $P(y) = 1/K$ for all $y$.
	From Lemma~\ref{lem:bayes_prob_softmax}, letting $\fens(x) = \fstd(x) + \frob(x)$, we have:
	\begin{equation}
		\label{eqn:first_in_proof_calibration-ensemble-optimal}
		P(y \mid \fstd(x), \frob(x)) = \softmax(\fens(x))_y,
	\end{equation}
	So this means that the ensemble prediction is the Bayes optimal given $(\fstd(x), \frob(x))$:
	\begin{equation}
		\pred(\fens(x)) = \argmax_y \fens(x)_y = \argmax_y \softmax(\fens(x))_y = \argmax_y P(y \mid \fstd(x), \frob(x)).
	\end{equation}
	But then from Lemma~\ref{lem:bayes-opt-is-optimal-appendix}, any other predictor which uses only $(\frob(x), \fstd(x))$ must have higher error.
	This completes the proof.

	Note that the inequality in the above proof is a strict inequality except in degenerate cases: as long as $\fstd$ and $\frob$ sometimes disagree in their predictions, and in some of these cases $\fstd$ assigns a higher probability to its predictions, and in some cases $\frob$ assigns a higher probability to its prediction, the inequalities will be strict inequalities.
\end{proof}

\subsection{Proof of Proposition~\ref{prop:supp_missing_ens_works}}

\newtheorem*{suppMissingEnsWorksProp}{Restatement of Proposition~\ref{prop:supp_missing_ens_works}}

\begin{suppMissingEnsWorksProp}
\suppMissingEnsWorksText{}
\end{suppMissingEnsWorksProp}

\begin{proof}
	We first note that errors are additive. That is, letting:
	\begin{equation}
		\Err(P, f) = \E[\pred(f(x)) \neq y] \mbox{, where }x, y \sim P,
	\end{equation}
	we have:
	\begin{equation}
		\Err(\alpha P_{\tau} + (1 - \alpha) P_0, f) = \alpha \Err(P_{\tau}, f) + (1 - \alpha) \Err(P_0, f)
	\end{equation}
	So it suffices to prove that the ensemble is better than the standard and robust models for $P_{\tau}$ and $P_0$ separately.

	\textbf{Suppressed features.}
	Let $\overline{\frob}(x) = \tau \frob(x)$ and $\overline{\fstd}(x) = \tau \fstd(x)$ be scaled versions of the standard and robust models.
	Definition~\ref{dfn:suppressed_spurious} implies that $\overline{\frob}$ and $\overline{\fstd}$ are calibrated.
	Since we assumed $P_{\tau}$ is balanced, by Proposition~\ref{prop:calibration-ensemble-optimal}, $\overline{\fens}$ given by $\overline{\fens(x)} = \tau \frob(x) + \tau \fstd(x)$ has optimal error on $P_{\tau}$.
	But for all $x$, the predictions of $\fens$ and $\overline{\fens}$ are the same (multiplying the outputs of a model by a constant does not change the predicted output, which is the $\argmax$).
	So $\fens$ also has optimal error on $P_{\tau}$:
	\begin{equation}
		\Err(P_{\tau}, \fens) \leq \Err(P_{\tau}, \fstd) \mbox{, and } \Err(P_{\tau}, \fens) \leq \Err(P_{\tau}, \frob)
	\end{equation}
	Note that these inequalities are strict inequalities except in degenerate cases: as long as $\fstd$ and $\frob$ sometimes disagree in their predictions, and in some of these cases $\fstd$ assigns a higher probability to its predictions, and in some cases $\frob$ assigns a higher probability to its prediction, the inequalities will be strict inequalities.

	\textbf{Missing spurious.}
	If $\fstd(x) = 0$ almost surely, then $\fens(x) = \frob(x) + \fstd(x) = \frob(x)$ almost surely.
	Furthermore, if $\fstd(x) = 0$ then its error is lower bounded by $1 - \max_y P_0(y)$.
	On the other hand, $\frob(x)$ is calibrated and therefore Bayes-optimal given $z = \frob(x)$ so from Lemma~\ref{cor:bayes-opt-better-trivial} (e.g., see the the discussion below the Lemma for more details) has error at most $1 - \max_y P_0(y)$.
	So we have:
	\begin{equation}
		\Err(P_0, \fens) = \Err(P_0, \frob) \leq \Err(P_0, \fstd)
	\end{equation}
	Note that the inequality is a strict inequality except in a degenerate case (where the probability that $\frob$ predicts for the most common class $\argmax_y P_0(y)$ is the same for all inputs).
\end{proof}

\subsection{Proof of Proposition~\ref{prop:anti_correlated_ens_fails}}

\newtheorem*{antiCorrelatedEnsFailsProp}{Restatement of Proposition~\ref{prop:anti_correlated_ens_fails}}

\begin{antiCorrelatedEnsFailsProp}
\antiCorrelatedEnsFailsText{}
\end{antiCorrelatedEnsFailsProp}

\begin{proof}
	Let $X, Y \sim \Pood$, and let $Z = (\fstd(X), \frob(X))$ be the predictions of the standard and robust models.
	Fix $z = (\fstd(x), \frob(x))$, and let $s = \fstd(x)$ and $r = \frob(x)$.
	We will analyze the errors for fixed $Z = z$ (showing that the robust model is better than the ensemble, which is better than the standard model).
	Since this is true for all $z$, we then use Lemma~\ref{lem:alt_error_appendix} (which is basically the law of total expectation), to get the desired result.

	\textbf{Bayes-opt classifier.}
	Recall that for some $\alpha, \beta > 0$, we have $\Padv(Y = y | \fstd(x)) = \softmax(-\beta \fstd(x))_y$ for all $x$ (note the minus sign), while $\Padv(Y = y \mid \frob(x)) = \softmax(\alpha \frob(x))_y$.
	Then, applying Lemma~\ref{lem:bayes_prob_softmax}, we have:
	\begin{equation}
		\Padv(y \mid (\fstd(x), \frob(x))) = \softmax(\alpha \frob(x) - \beta \fstd(x))_y.
	\end{equation}
	Rewriting this in terms of $z, r, s$, we have:
	\begin{equation}
		\label{eqn:bayes-opt-adv-spur}
		\Padv(y \mid z) = \softmax(\alpha r - \beta s)_y.
	\end{equation}

	% Latex macro note:
	% We use \jens, \jstd, \jrob to denote the indices selected by the ensemble, standard, and robust model
	\textbf{Ensemble vs. robust classifier.}
	Let $\jrob = \argmax_y r_y$ be the robust model's prediction, and $\jens = \argmax_y (r+s)_y$ be the ensemble model's prediction.
	Because $\jrob$ is the $\argmax$ of $r$, we have:
	\begin{equation}
		\label{eqn:rjrobvsjens_appendix}
		r_{\jrob} \geq r_{\jens}.
	\end{equation}
	Because $\jens$ is the $\argmax$ of $r+s$, we have:
	\begin{equation}
		\label{eqn:ensjrobvsjens_appendix}
		r_{\jens} + s_{\jens} \geq r_{\jrob} + s_{\jrob}.
	\end{equation}
	Taking the negation of this, we get:
	\begin{equation}
		\label{eqn:ensjrobvsjens_negation_appendix}
		-r_{\jrob} - s_{\jrob} \geq -r_{\jens} - s_{\jens}.
	\end{equation}
	Adding $\beta$ times Inequality~\ref{eqn:ensjrobvsjens_negation_appendix} to $(\alpha + \beta)$ times Inequality~\ref{eqn:rjrobvsjens_appendix}, we get:
	\begin{equation}
		\alpha r_{\jrob} - \beta s_{\jrob} \geq \alpha r_{\jens} - \beta s_{\jens}.
	\end{equation}
	Since $\softmax$ is monotonic, we have:
	\begin{equation}
		\softmax(\alpha r - \beta s)_{\jrob} \geq \softmax(\alpha r - \beta s)_{\jens}.
	\end{equation}
	But from Equation~\ref{eqn:bayes-opt-adv-spur} the LHS is the same as the robust model's probability of getting the label correct, and the RHS is the same as the ensemble's probability of getting the label correct:
	\begin{equation}
		\Padv(Y = \jrob \mid Z=z) \geq \Padv(Y = \jens \mid Z=z).
	\end{equation}
	Taking negations (to get the error), and then the expectation over $Z=z$, we get (note that below we write the error, which is why the sign is now flipped):
	\begin{equation}
		\Errood(\fens) \geq \Errood(\frob).
	\end{equation}
	Which is what we wanted to show.
	
	\textbf{Ensemble vs. standard classifier.}
	The argument is fairly analogous to the previous case, with some minor differences in the algebra in the first part.
	Let $\jstd = \argmax_y s_y$ be the standard model's prediction.
	Because $\jstd$ is the $\argmax$ of $s$, we have:
	\begin{equation}
		\label{eqn:sjstdvsjens_appendix}
		s_{\jstd} \geq s_{\jens}.
	\end{equation}
	Taking the negation of this, we get:
	\begin{equation}
		\label{eqn:sjstdvsjens_negation_appendix}
		-s_{\jens} \geq -s_{\jstd}.
	\end{equation}
	Because $\jens$ is the $\argmax$ of $r+s$, we have:
	\begin{equation}
		\label{eqn:ensjstdvsjens_appendix}
		r_{\jens} + s_{\jens} \geq r_{\jstd} + s_{\jstd}.
	\end{equation}
	Adding $\alpha$ times Inequality~\ref{eqn:ensjstdvsjens_appendix} with $(\alpha + \beta)$ times Inequality~\ref{eqn:sjstdvsjens_negation_appendix}, we get:
	\begin{equation}
		\alpha r_{\jens} - \beta s_{\jens} \geq \alpha r_{\jstd} - \beta s_{\jstd}.
	\end{equation}
	The rest of this step is the same as in the comparison between the ensemble and the robust model.
	Since $\softmax$ is monotonic, we have:
	\begin{equation}
		\softmax(\alpha r - \beta s)_{\jens} \geq \softmax(\alpha r - \beta s)_{\jstd}.
	\end{equation}
	But from Equation~\ref{eqn:bayes-opt-adv-spur} the LHS is the same as the robust model's probability of getting the label correct, and the RHS is the same as the ensemble's probability of getting the label correct:
	\begin{equation}
		\Padv(Y = \jens \mid Z=z) \geq \Padv(Y = \jstd \mid Z=z).
	\end{equation}
	Taking negations (to get the error), and then the expectation over $Z=z$, we get (note that below we write the error, which is why the sign is now flipped):
	\begin{equation}
		\Errood(\fstd) \geq \Errood(\fens).
	\end{equation}
	Which is what we wanted to show.

	% Combining these two inequalities (i.e., take the negative for Inequality~\ref{eqn:rjrobvsjens_appendix} to get $-r_{\jens} \geq -r_{\jrob}$ and then add that to Inequality~\ref{eqn:ensjrobvsjens_appendix}), we get:
	% \begin{equation}
	% 	\label{eqn:sjrobvsjens_appendix}
	% 	s_{\jens} \geq s_{\jrob}.
	% \end{equation}
	% Negating Inequality~\ref{eqn:sjrobvsjens_appendix} (to get $-s_{\jrob} \geq -s_{\jens}$) and adding to Inequality

	% Let's consider a more general case where $\Pood$ is balanced and satisfies the following conditions, for all $y \in [K]$:
	% \begin{align}
	% 	P(y \mid \fstd(x)) &= \softmax(\alpha \fstd(x))_y \\
	% 	P(y \mid \rob(x)) &= \softmax(\beta \rob(x))_y
	% \end{align}
	% Then from Lemma~\ref{lem:bayes_prob_softmax}, we have for all $y \in [K]$:
	% \begin{equation}
	% 	P(y \mid \fstd(x), \frob(x)) = \softmax(\fstd(x) + \frob(x))_y
	% \end{equation}

	% Let $X, Y \sim P$ and let $Z = (\fstd(X), \fstd(Y))$.
	% Fix some $z$. We define the accuracy conditioned on $z$ as:
	% % The predictions of the standard, robust, ensemble, can be written as a function of z
	% % I.e., acc(z, P, yhat) P(Y = \yhat(z) | Z=z)
	% % Can define accuracy of a model
	% % That's just acc(z, P, pred(u)
	% % Suppose standard model predicts a, robust model b, ensemble c
	% % Make it clear this is shorthand
	% % We get some inequalities
	% % We can show acc(z, P, c) >= acc(z, P, a) and
	% % acc(z, P, c) >= acc(z, P, b)
	% % Now need to link this to overall error. The point is that a, b, c are implicitly functions of z. So now can take expectation
\end{proof}

\paragraph{Dealing with class imbalance.}
Lemma~\ref{lem:bayes_prob_softmax}, Equation~\ref{eqn:general_combination_lem_bayes_prob_softmax} shows how to combine models in general, if the class-balanced assumption does not hold. Note the additional ``$-m$'' term. Here, the (marginal) probability of each class is defined in Equation~\ref{eqn:general_combination_lem_bayes_prob_softmax}.

(ID Analysis) Then, the ``Proof of Proposition~\ref{prop:calibration-ensemble-optimal}'' is identical for the general case, we just need to set $\fens(x) = \fstd(x) + \frob(x) - m$ on the first line. Equation~\ref{eqn:first_in_proof_calibration-ensemble-optimal} then follows from Lemma~\ref{lem:bayes_prob_softmax}, and the rest of the proof is identical.

(OOD Analysis) The OOD results, Proposition~\ref{prop:supp_missing_ens_works} and~\ref{prop:anti_correlated_ens_fails}, follow if the class marginal distributions match up between ID and OOD, so $\Pid(Y=y) = \Pood(Y=y)$. If the distribution over classes changes substantially, then ensembles can possibly do worse than the robust model.

% \subsection{Proof of Example~\ref{ex:imbalance-ensemble-fails}}

% \newtheorem*{imbalancedEnsembleFailsExample}{Restatement of Example~\ref{ex:imbalance-ensemble-fails}}

% \begin{imbalancedEnsembleFailsExample}
% \imbalancedEnsembleFailsText{}
% \end{imbalancedEnsembleFailsExample}
\section{More information on experiments}

\subsection{Additional details on datasets}
\label{sec:more-info-datasets-appendix}

Here we describe the robustness interventions and datasets in more detail.

\textbf{Robustness interventions}:
\begin{enumerate}
	\item In-N-Out~\citep{xie2021innout}. Many datasets contain a core input $x$ (image or time series data), and metadata $z$ (e.g., location or climate data).~\citet{xie2021innout} show that using the metadata (in addition to $x$) improves accuracy in-distribution (ID), but hurts accuracy out-of-distribution.~\citet{xie2021innout} consider a standard model that takes in both the core inputs and metadata to predict the target, and a robust model that only takes in the core inputs and does some additional pretraining. We use official checkpoints from their CodaLab worksheet \url{https://worksheets.codalab.org/worksheets/0x2613c72d4f3f4fbb94e0a32c17ce5fb0}, and compare to the results tagged as ``In-N-Out'' on each dataset. They also show results after doing additional self-training on (unlabeled) OOD data, but we do not compare to this because 1. OOD data is assumed to be unavailable in our setting, and 2. if OOD unlabeled data is available, we can also start from \calens{} and do additional self-training.
% They call these the `aux-in' and `aux-out' models respectively.
	\item Lightweight fine-tuning~\citep{kumar2022finetuning}: When adapting a pretrained model to an ID dataset, typically all the model parameters are fine-tuned. Recent works show that tuning only parts of the model can often do better OOD even though the ID performance is worse~\citep{li2021prefix,houlsby2019parameter}. On four distribution shift datasets, we take checkpoints from~\citet{kumar2022finetuning} where the standard model starts from a pretrained initialization and fine-tunes all parameters on an ID dataset, and the robust model only learns the top linear `head' layer.
	\item Zero-shot language prompting:~\citet{radford2021clip} pretrain a model on a large multi-modal language and vision dataset. The model can then predict the label of an image by comparing the image embedding, with the language embedding for prompts such as `photo of an apple' or `photo of a banana'. They show that this zero-shot language prompting approach (robust model) can be much more accurate OOD than the traditional method of fine-tuning the entire model (standard model), although ID accuracy of the robust model is worse. We use model checkpoints and datasets from~\citet{radford2021clip}.
	\item Group distributionally robust optimization (DRO)~\citep{sagawa2020group}: Standard ERM models often latch on to spurious correlations in a dataset, such as image background color, or the occurrence of certain words in a sentence. Group DRO essentially upweights examples where this spurious correlation is not present.
	The original formulation in~\citet{sagawa2020group} assumes the spurious correlations are annotated, but newer variants~\citep{liu2021jtt} can work even without these annotations.
	\item CORAL~\citep{sun2016deep} aims to align feature representations across different domains, by penalizing differences in the means and covariances
	of the feature distributions.
	The hope is that this generalizes better to OOD domains.
 \end{enumerate}

We consider three types of \natshifts{} (geography shifts, subpopulation shifts, style shifts), and we also consider adversarial spurious shifts.

\textbf{Geography shifts.} In geography shifts the ID data comes from some locations, and the OOD data comes from a different set of locations. One motivation is that in many developing areas training data may be unavailable because of monetary constraints~\citep{jean2016combining}.
\begin{enumerate}
	\item \textbf{LandCover}~\citep{russwurm2020meta}: The goal is to classify a satellite time-series into one of 6 land types (e.g., "grassland", "savannas"). The ID data contains time-series from outside Africa, and the OOD data consists of time-series from Africa. We take model checkpoints from~\citet{xie2021innout} where they use the In-N-Out intervention---the core feature $x$ is time series data measured by Nasa's MODIS satellite~\citep{modis2015landcover}, and the spurious metadata $z$ consists of climate data (e.g., temperature) at that location. We use the ID and OOD dataset splits defined by~\citet{xie2021innout}.
	\item \textbf{Cropland}~\citep{wang2020weakly}: The goal is to predict whether a satellite image is of a cropland or not. The ID dataset contains images from Iowa, Missouri, and Illinois, and the OOD dataset contains images from Indiana and Kentucky. We take model checkpoints from~\citet{xie2021innout} where they use the In-N-Out intervention---the core feature $x$ is an RGB satellite image, and the spurious metadata $z$ consists of location coordinates and vegetation bands. We use the ID and OOD dataset splits defined by~\citet{xie2021innout}.
	\item \textbf{iWildCam}~\citep{beery2020iwildcam,koh2021wilds}: The goal is to classify the species of an animal given a photo taken by a camera placed in the wild (e.g., in a forest). The ID dataset consists of photos taken by over 200 cameras, and the OOD dataset consists of photos taken by held-out cameras. We use the splits by~\citet{koh2021wilds}. We take model checkpoints from~\citet{koh2021wilds}, where the standard model is trained via standard empirical risk minimization (ERM), and the robust model is trained via CORAL. The model checkpoints were taken from \url{https://worksheets.codalab.org/worksheets/0x036017edb3c74b0692831fadfe8cbf1b}.
\end{enumerate}
\ak{Optionally move the discussion of what's core and spurious to the Appendix}

\textbf{Subpopulation shifts.} In subpopulation shifts, the ID data contains a few sub-categories (e.g., black bear and sloth bear), and the OOD data contains different sub-categories (e.g., brown bears and polar bears) or the same parent category (e.g., bears). For both datasets below, we take model checkpoints from~\citet{kumar2022finetuning} where they use the lightweight fine-tuning intervention, starting from a MoCo-v2 ResNet-50 model pretrained on unlabeled ImageNet images. The datasets are from~\citet{santurkar2020breeds}.
\begin{enumerate}
	\item \textbf{Living-17}~\citep{santurkar2020breeds}: the goal is to classify an image as one of 17 animal categories such as ``bear'' - the ID dataset contains images of black bears and sloth bears and the OOD dataset has images of brown bears and polar bears. 
	\item \textbf{Entity-30}~\citep{santurkar2020breeds}: similar to Living-17, except the goal is to classify an image as one of 30 entity categories such as ``food'', ``motor vehicle'', and ``index''.
\end{enumerate}
\ak{Optionally move the discussion of what checkpoints we used to the Appendix.}

\textbf{Style shifts.} In style shifts, the ID data contains data in a certain style (e.g., sketches), and the OOD data contains data in a different style (e.g., real photos, renditions). 
\begin{enumerate}
	\item \textbf{DomainNet}~\citep{peng2019moment}: a standard domain adaptation dataset. Here, our ID dataset contains ``sketch'' images (e.g., drawings of apples, elephants, etc), and the OOD dataset contains ``real'' photos of the same categories. We take model checkpoints from~\citet{kumar2022finetuning} where they use the lightweight fine-tuning intervention, starting from a CLIP ResNet-50 model.
	\item \textbf{CelebA}~\citep{liu2015deep}: the goal is to classify a portrait of a face as ``male'' or ``female'' - the ID dataset contains images of people without hats, and the OOD dataset contains images of people wearing hats (some facial features might be ``suppressed'' or ``missing'' with hats). We take model checkpoints from~\citet{xie2021innout} where they use the In-N-Out intervention---the core feature $x$ is the RGB image, and the spurious metadata $z$ consists of 7 attribute tags annotated in the dataset (e.g., presence of makeup, beard).
	\item \textbf{CIFAR->STL}: standard domain adaptation dataset~\citep{french2018selfensembling}, where the ID is CIFAR-10~\citep{krizhevsky2009learningmultiple}, and the OOD is STL~\citep{coates2011stl10}. The task is to classify an image into one of 10 categories such as ``dog'', ``cat'', or ``airplane''. We take model checkpoints from~\citet{kumar2022finetuning} where they use the lightweight fine-tuning intervention, starting from a MoCo-v2 ResNet-50 model pretrained on unlabeled ImageNet images. Note that STL has one class (monkey) that is not present in CIFAR-10, so we omit this class when testing on STL.
	\item \textbf{ImageNet}~\citep{russakovsky2015imagenet}: a large scale dataset where the goal is to classify an image into one of 1000 categories. We use the zero-shot language prompting intervention using a CLIP ViT-B/16 vision transformer model. We evaluate on 3 standard OOD datasets: \textbf{ImageNetV2}~\citep{recht2019doimagenet},\textbf{ImageNet-R}~\citep{hendrycks2020many}, and \textbf{ImageNet-Sketch}~\citep{wang2019learningrobust}. \ak{The zero-shot model checkpoint is taken from CLIP paper, and the fine-tuned model checkpoint is taken from LP-FT paper}
\end{enumerate}

\textbf{Adversarial spurious shifts.} In adversarial spurious shifts, the ID dataset contains a feature that is correlated with a label, but this correlation is flipped OOD.
% For example, waterbirds is explicitly constructed so that ``water'' backgrounds are correlated with ``waterbird'' labels in the ID, but anti-correlated OOD.
\begin{enumerate}
	\item \textbf{Waterbirds}~\citep{sagawa2020group}: The goal is to classify an image as a ``waterbird'' or ``landbird''. The dataset is synthetically constructed to have adversarially spurious features: ``water'' backgrounds are correlated with ``waterbird'' labels in the ID, but anticorrelated OOD. We use checkpoints from~\citet{jones2021selective} where they use the group DRO intervention.
	\item \textbf{MNLI}~\citep{williams2018broad}: The goal is to predict whether a hypothesis is entailed, contradicted by, or neutral to an associated premise. We use the splits in~\citet{sagawa2020group}---they partition the dataset so that in-distribution ``negation'' words ``nobody'', ``no'', ``never'', and ``nothing'' are correlated with the contradiction label, however in the OOD dataset these words are anticorrelated with the contradiction label. We use checkpoints from~\citet{jones2021selective} where they use the group DRO intervention.
	\item \textbf{CivilComments}~\citep{borkan2019nuanced}: The goal is to predict whether a comment is toxic or not. We use the splits in~\citet{sagawa2020group}---they partition the dataset where in the ID split mentions of a Christian identity are correlated with non-toxic comments, but in the OOD split mentions of a Christian identity are correlated with a toxic comment. We use checkpoints from~\citet{jones2021selective} where they use the group DRO intervention. CivilComments is also used in~\citet{koh2021wilds}.
\end{enumerate} 

\subsection{Additional details on comparison with self-training / In-N-Out}
\label{sec:more-info-datasets-appendix-in-n-out}

We give additional details on our comparisons to self-training (In-N-Out~\citep{xie2021innout}), which we report in Table~\ref{tab:self_train_results} in Section~\ref{sec:experiments}.
Our standard model corresponds to the `aux-in' model in In-N-Out terminology---the aux-in model takes in the core input and the auxiliary input, and trains via standard empirical risk minimization.
Our robust model corresponds to  the `aux-out' model in In-N-Out terminology---the aux-out model does a pretraining step where it tries to predict the auxiliary input from the core input, on both ID and OOD (unlabeled) data. Note that the aux-out model does not feed in the auxiliary inputs into the model at all.

We compare with their In-N-Out method which uses self-training on ID data to mitigate ID-OOD tradeoffs. Note that~\citet{xie2021innout} also consider an additional self-training step on both ID and OOD unlabeled data, which they call In-N-Out + repeated ST. We do not compare with that since it uses additional OOD data in the method, which \calens{} do not.

We download checkpoints from the official CodaLab implementation: \url{https://worksheets.codalab.org/worksheets/0x2613c72d4f3f4fbb94e0a32c17ce5fb0}. For each dataset, we take the first 3 trials (checkpoints) of aux-inputs (standard model) and aux-outputs (robust model). We then compare the results of applying out method (\calens{}) with the results listed in the CodaLab worksheet for In-N-Out (but not In-N-Out + repeated self-training which also does self-training on OOD unlabeled data).

\subsection{Per-dataset results on ensembling ablations}
\label{sec:per-dataset-ensemble-ablations}

\textbf{Different ways of ensembling models.} In Section~\ref{sec:experiments-how-ensemble} we ablated calibrated ensembles with ``tuned'' ensembles where the ensemble weights are tuned on in-distribution validation data, and with vanilla ensembles.
The main takeaway was the calibrated ensembles outperform these other ways of ensembling the standard and robust models, when tested on out-of-distribution (OOD) test examples.
Here, we show per-dataset results both ID (Table~\ref{tab:id_tuned}) and OOD (Table~\ref{tab:ood_tuned}).

Now, we describe what each of the ensembling ablations are more formally, building on the notation in Section~\ref{sec:methods}.

\textbf{Logits.} Here, we add up the logits of the two models, which is equivalent to multiplying their probabilities. We do not apply a calibration step.
\begin{equation}
	\fens(x) = \fstd(x) + \frob(x).
\end{equation}
\textbf{Probs.} Here, we average the probabilities of the two models, without any calibration step. The purpose of the $\log$ below is to convert back to logit space.
\begin{equation}
	\fens(x) = \log\Big(\softmax(\fstd(x)) + \softmax(\frob(x))\Big).
\end{equation}
We add the model probabilities inside the $\log$, instead of averaging them---but this is equivalent to averaging the probabilities after taking the $\softmax$ because $\softmax$ normalizes it back into probabilities. Formally, with a bit of algebra, we have:
\begin{equation}
	\softmax(\fens(x)) = \frac{1}{2}\Big(\softmax(\fstd(x)) + \softmax(\frob(x))\Big).
\end{equation}

\textbf{Tuned Logits.} Here, we take a weighted average of the logits of the models. The weight $\alpha$ is selected from $\{0.0, 0.1, 0.2, \ldots, 1.0\}$, and is chosen to maximize accuracy on the \emph{in-distribution} validation set, which is the same data we calibrate on. The main point of this ablation is that other sensible ways of using the ID validation data besides calibration, do not do as well as \calens{}.
\begin{equation}
	\fens(x) = \alpha \fstd(x) + (1 - \alpha) \frob(x).
\end{equation}

\textbf{Tuned Probs.} Here, we take a weighted average of the probabilities of the two models, where the weight $\alpha$ is selected from $\{0.0, 0.1, 0.2, \ldots, 1.0\}$, and is chosen to maximize accuracy on the \emph{in-distribution} validation set.
\begin{equation}
	\fens(x) = \log\Big(\alpha \cdot \softmax(\fstd(x)) + (1 - \alpha) \cdot \softmax(\frob(x))\Big).
\end{equation}

\textbf{Calibrated Logits.} Here, we calibrate the models and then add up their logits, where $\Tstd$ and $\Trob$ are selected on ID validation data so that the standard and robust models' confidences and accuracies match up. (Equation~\ref{eqn:tstd_trob_eqn}, Section~\ref{sec:methods}).
\begin{equation}
	\fens(x) = \frac{\fstd(x)}{\Tstd} + \frac{\frob(x)}{\Trob}.
\end{equation}

\textbf{Calibrated Probs.} Here, we calibrate the models and then add up their probabilities. This is precisely what we do in Section~\ref{sec:methods} in the main paper, and what we report for all the other results besides the ablations.
\begin{align}
\fens(x) = \log\Big(\softmax\Big(\frac{\fstd(x)}{\Tstd}\Big) + \softmax\Big(\frac{\frob(x)}{\Trob}\Big)\Big).
\end{align}

\textbf{Ensembling two standard or two robust models.} In Section~\ref{sec:experiments-how-ensemble}, 
We also compared calibrated ensembles (of one standard and one robust model) with ensembles of two standard models, and ensembles of two robust models, where for a fair comparison all models are calibrated.
We ran this ablation on 6 of the \numtotal{} datasets (Entity-30, DomainNet, CIFAR$\to$STL, Living-17, Landcover, Cropland, and CelebA) because it requires multiple standard and multiple robust models, which were not available or very expensive to run on large datasets like ImageNet.
Calibrated ensembles get an average ID accuracy of \calaccidSeven{}\% (vs. \robrobaccidSeven{}\% for a robust-robust ensemble and \stdstdaccidSeven{}\% for a standard-standard ensemble), and an average OOD accuracy of \calaccoodSeven{}\% (vs. \robrobaccoodSeven{}\% for a robust-robust ensemble and \stdstdaccoodSeven{}\% for a standard-standard ensemble). We show per-dataset results in Table~\ref{tab:id_std_std_rob_rob} (ID) and Table~\ref{tab:ood_std_std_rob_rob} (OOD).
We show per-dataset results both ID (Table~\ref{tab:id_std_std_rob_rob}) and OOD (Table~\ref{tab:ood_std_std_rob_rob}).

\begin{table*}[t]
\caption{
\emph{ID} accuracies: The in-distribution accuracies of calibrated ensembles, tuned ensembles, and vanilla ensembles are very close (within confidence intervals), so any of these methods are acceptable if we are looking at in-distribution accuracy. However, they perform quite differently when it comes to OOD accuracy (Table~\ref{tab:ood_tuned}).
}
\label{tab:id_tuned}
\vskip 0.15in
\begin{center}

\begin{tabular}{cccccccc}
\toprule
 & Ent30 & DomNet & CIFAR10 & Liv17 & Land & Crop & CelebA\\
\midrule
Logits & 93.7 (0.1) & 89.3 (0.6) & \textbf{97.3 (0.1)} & \textbf{97.1 (0.2)} & \textbf{77.4 (0.1)} & 95.5 (0.1) & 93.4 (0.6)\\
Probs & \textbf{93.7 (0.1)} & 89.1 (0.4) & \textbf{97.3 (0.1)} & \textbf{97.1 (0.2)} & \textbf{77.4 (0.2)} & 95.5 (0.1) & 93.4 (0.6)\\
Tuned Logits & \textbf{93.8 (0.0)} & \textbf{91.3 (0.2)} & \textbf{97.4 (0.1)} & 97.1 (0.1) & \textbf{77.3 (0.4)} & \textbf{95.6 (0.1)} & \textbf{94.8 (0.2)}\\
Tuned Probs & 93.8 (0.1) & 90.6 (0.7) & \textbf{97.4 (0.1)} & \textbf{97.2 (0.1)} & 77.1 (0.3) & 95.5 (0.1) & \textbf{95.0 (0.2)}\\
Calibrated Logits & 93.7 (0.1) & \textbf{91.1 (0.4)} & 97.2 (0.1) & \textbf{97.2 (0.2)} & 77.2 (0.2) & \textbf{95.6 (0.1)} & \textbf{94.5 (0.5)}\\
Calibrated Probs & 93.7 (0.1) & \textbf{91.2 (0.7)} & 97.2 (0.1) & \textbf{97.2 (0.2)} & 77.2 (0.2) & \textbf{95.6 (0.1)} & \textbf{94.5 (0.5)}\\
\bottomrule
\end{tabular}
\vspace{1.2mm}
\newline
\begin{tabular}{cccccc}
\toprule
 & ImageNet & iWildCam & MNLI & Waterbirds & Comments\\
\midrule
Logits & 82.1 (-) & \textbf{84.2 (-)} & \textbf{82.9 (-)} & 90.1 (-) & 90.4 (-)\\
Probs & 82.1 (-) & 83.9 (-) & \textbf{82.9 (-)} & 90.1 (-) & 90.4 (-)\\
Tuned Logits & \textbf{82.7 (-)} & \textbf{84.1 (-)} & \textbf{83.0 (-)} & \textbf{93.2 (-)} & \textbf{92.7 (-)}\\
Tuned Probs & 82.3 (-) & 83.9 (-) & \textbf{83.0 (-)} & \textbf{93.2 (-)} & \textbf{92.6 (-)}\\
Calibrated Logits & 82.0 (-) & \textbf{84.3 (-)} & 82.8 (-) & 92.9 (-) & 91.4 (-)\\
Calibrated Probs & 82.0 (-) & 84.0 (-) & \textbf{82.8 (-)} & 92.9 (-) & 91.4 (-)\\
\bottomrule
\end{tabular}

% \begin{tabular}{ccccccc}
% \toprule
%  & Ent30 & DomNet & CIFAR10 & Land & Crop & ImNet\\
% \midrule
% Logits & 93.7 (0.1) & 89.3 (0.6) & 97.3 (0.1) & 77.4 (0.1) & 95.5 (0.1) & 80.9 (-)\\
% Probs & 93.7 (0.1) & 89.1 (0.4) & 97.3 (0.1) & 77.4 (0.2) & 95.5 (0.1) & 81.0 (-)\\
% Tuned Logits & 93.8 (0.0) & 91.3 (0.2) & 97.4 (0.1) & 77.3 (0.4) & 95.6 (0.1) & 81.7 (-)\\
% Tuned Probs & 93.8 (0.1) & 90.6 (0.7) & 97.4 (0.1) & 77.1 (0.3) & 95.5 (0.1) & 81.3 (-)\\
% Calibrated Logits & 93.7 (0.1) & 91.1 (0.4) & 97.2 (0.1) & 77.2 (0.2) & 95.6 (0.1) & 81.0 (-)\\
% Calibrated Probs & 93.7 (0.1) & 91.2 (0.7) & 97.2 (0.1) & 77.2 (0.2) & 95.6 (0.1) & 81.1 (-)\\
% \bottomrule
% \end{tabular}
\end{center}
\vskip -0.1in
\end{table*}

\begin{table*}[t]
\caption{
\emph{OOD} accuracies: calibrated ensembles outperform vanilla ensembles and even tuned ensembles where the combination weights are tuned to maximize in-distribution accuracy. Averaged across the datasets, calibrated ensembles get an OOD accuracy of \calaccood\%, while tuned ensembles get an accuracy of \tunedaccood\%. The in-distribution accuracies of the methods are very close (within 0.2\% of each other).
}
\label{tab:ood_tuned}
\vskip 0.15in
\begin{center}

\begin{tabular}{cccccccc}
\toprule
 & Ent30 & DomNet & CIFAR10 & Liv17 & Land & Crop & CelebA\\
\midrule
Logits & \textbf{64.9 (0.3)} & 75.7 (1.2) & 87.3 (0.2) & \textbf{81.8 (0.4)} & \textbf{60.5 (0.8)} & \textbf{90.9 (0.2)} & \textbf{76.9 (0.9)}\\
Probs & 64.6 (0.4) & 78.7 (1.3) & 87.2 (0.2) & \textbf{81.8 (0.4)} & 59.5 (1.0) & \textbf{90.9 (0.2)} & \textbf{76.9 (0.9)}\\
Tuned Logits & \textbf{64.6 (0.6)} & 86.3 (0.6) & 85.7 (0.9) & 80.8 (0.7) & 58.7 (1.2) & \textbf{87.3 (5.7)} & \textbf{77.5 (1.3)}\\
Tuned Probs & 62.8 (0.7) & \textbf{86.9 (0.2)} & 85.0 (1.3) & 81.6 (0.5) & 58.7 (2.2) & \textbf{86.8 (5.5)} & \textbf{77.6 (1.7)}\\
Calibrated Logits & \textbf{65.0 (0.4)} & 84.4 (0.3) & \textbf{87.5 (0.2)} & \textbf{82.0 (0.4)} & \textbf{61.2 (0.8)} & \textbf{91.3 (0.8)} & \textbf{77.6 (1.2)}\\
Calibrated Probs & \textbf{64.7 (0.5)} & 86.1 (0.2) & 87.3 (0.2) & \textbf{82.2 (0.6)} & \textbf{60.8 (0.8)} & \textbf{91.3 (0.8)} & \textbf{77.6 (1.2)}\\
\bottomrule
\end{tabular}
\vspace{1.2mm}
\newline
\begin{tabular}{cccccccc}
\toprule
 & ImNet-R & ImNet-V2 & ImNet-Sk & iWildCam & MNLI & Waterbirds & Comments\\
\midrule
Logits & 73.1 (-) & \textbf{73.7 (-)} & 52.1 (-) & \textbf{66.2 (-)} & 73.1 (-) & 66.9 (-) & \textbf{76.0 (-)}\\
Probs & 77.5 (-) & 73.4 (-) & 52.0 (-) & 65.3 (-) & 72.4 (-) & 66.9 (-) & \textbf{76.0 (-)}\\
Tuned Logits & 64.7 (-) & \textbf{73.6 (-)} & 47.9 (-) & 66.0 (-) & 68.0 (-) & \textbf{88.1 (-)} & 60.3 (-)\\
Tuned Probs & 64.0 (-) & 72.6 (-) & 45.5 (-) & 65.3 (-) & 69.4 (-) & \textbf{88.1 (-)} & 61.5 (-)\\
Calibrated Logits & 73.7 (-) & \textbf{73.6 (-)} & \textbf{52.3 (-)} & \textbf{66.1 (-)} & \textbf{73.6 (-)} & 81.1 (-) & 71.8 (-)\\
Calibrated Probs & \textbf{77.9 (-)} & 73.2 (-) & \textbf{52.3 (-)} & \textbf{66.3 (-)} & 73.2 (-) & 81.1 (-) & 71.8 (-)\\
\bottomrule
\end{tabular}

% \begin{tabular}{ccccccc}
% \toprule
%  & Ent30 & DomNet & STL & Land & Crop & ImNet-R\\
% \midrule
% Vanilla Logits & \textbf{64.9 (0.3)} & 75.7 (1.2) & \textbf{87.3 (0.2)} & \textbf{60.5 (0.8)} & \textbf{90.9 (0.2)} & 72.2 (-)\\
% Vanilla Probs & \textbf{64.6 (0.4)} & 78.7 (1.3) & 87.2 (0.2) & 59.5 (1.0) & \textbf{90.9 (0.2)} & \textbf{77.4 (-)}\\
% Tuned Logits & \textbf{64.6 (0.6)} & 86.3 (0.6) & 85.7 (0.9) & 58.7 (1.2) & \textbf{87.3 (5.7)} & 63.1 (-)\\
% Tuned Probs & 62.8 (0.7) & \textbf{86.9 (0.2)} & 85.0 (1.3) & 58.7 (2.2) & \textbf{86.8 (5.5)} & 63.8 (-)\\
% Calibrated Logits & \textbf{65.0 (0.4)} & 84.4 (0.3) & \textbf{87.5 (0.2)} & \textbf{61.2 (0.8)} & \textbf{91.3 (0.8)} & 71.7 (-)\\
% Calibrated Probs & \textbf{64.7 (0.5)} & 86.1 (0.2) & \textbf{87.3 (0.2)} & \textbf{60.8 (0.8)} & \textbf{91.3 (0.8)} & \textbf{77.1 (-)}\\
% \bottomrule
% \end{tabular}
\end{center}
\vskip -0.1in
\end{table*}

%%%%%%%%%%%%%% Std-std, rob-rob, vs. calibrated ensemble (ID) %%%%%%%%%%%%%%%%%%%
\begin{table*}[t]
\caption{
\emph{ID} accuracies: Calibrated ensembles (one standard and one robust model) achieve comparable or better performance to Standard ensembles (ensemble of two calibrated standard models) and Robust ensembles (ensemble of two calibrated robust models).
}
\label{tab:id_std_std_rob_rob}
\vskip 0.15in
\begin{center}

\begin{tabular}{cccccccc}
\toprule
 & Ent30 & DomNet & CIFAR10 & Liv17 & Land & CelebA\\
\midrule
Std Ensemble & \textbf{94.0 (0.0)} & 86.3 (0.4) & \textbf{97.7 (0.1)} & \textbf{97.0 (0.3)} & \textbf{77.9 (0.1)} & 91.7 (0.4)\\
Rob Ensemble & 90.9 (0.2) & 89.3 (0.3) & 92.0 (0.0) & \textbf{97.1 (0.1)} & 73.4 (0.2) & \textbf{95.2 (0.4)}\\
Cal ensemble & 93.7 (0.1) & \textbf{91.2 (0.7)} & 97.2 (0.1) & \textbf{97.2 (0.2)} & 77.2 (0.2) & 94.5 (0.5)\\
\bottomrule
\end{tabular}

\end{center}
\vskip -0.1in
\end{table*}

%%%%%%%%%%%%%% Std-std, rob-rob, vs. calibrated ensemble (OOD) %%%%%%%%%%%%%%%%%%%
\begin{table*}[t]
\caption{
\emph{OOD} accuracies: Calibrated ensembles (one standard and one robust model) achieve comparable or better performance to Standard ensembles (ensemble of two calibrated standard models) and Robust ensembles (ensemble of two calibrated robust models).
}
\label{tab:ood_std_std_rob_rob}
\vskip 0.15in
\begin{center}

\begin{tabular}{cccccccc}
\toprule
 & Ent30 & DomNet & CIFAR10 & Liv17 & Land & CelebA\\
\midrule
Std Ensemble & 61.7 (0.2) & 57.9 (0.2) & 83.5 (0.2) & 78.6 (0.4) & 57.5 (0.7) & 73.7 (1.1)\\
Rob Ensemble & 63.8 (0.4) & \textbf{87.5 (0.1)} & 85.1 (0.1) & \textbf{82.4 (0.1)} & \textbf{60.5 (1.4)} & \textbf{78.0 (0.6)}\\
Cal ensemble & \textbf{64.7 (0.5)} & 86.1 (0.2) & \textbf{87.3 (0.2)} & \textbf{82.2 (0.6)} & \textbf{60.8 (0.8)} & \textbf{77.6 (1.2)}\\
\bottomrule
\end{tabular}

\end{center}
\vskip -0.1in
\end{table*}

\subsection{Per-dataset results on calibration and confidence}
\label{sec:per-dataset-calibration-appendix}

\textbf{Relative confidence can be incorrect.}
We measure the confidence of a model $f$ on a distribution $P$ as $\mbox{conf}(f, P) = \E_{x \sim P}[\max_i f(x)_i]$.
Even if the models are not calibrated OOD, one intuitive intuition for why calibrated ensembles work is that that robust model has higher confidence OOD, so that the ensemble primarily uses the (more accurate) robust model's predictions OOD.
\ak{maybe simplify. even if the robust model is more confident on average that doesn't mean we get best of both worlds. so maybe just directly say: on OOD data, the robust model is typically more confident than the standard model, which is reasonable since the robust model is also more accurate.}
However, on the remote sensing dataset Landcover we find that the robust model is 6\% \emph{less confident} on OOD data than the standard model even though the robust model is 5\% \emph{more accurate} OOD than the standard model.
Interestingly, calibrated ensembles are able to combine the models in a more fine-grained way to get the best of both worlds, which is captured in our stylized setting in Section~\ref{sec:analysis}.
We show the average confidence of the standard and robust models for each dataset ID (Table~\ref{tab:id_conf}) and OOD (Table~\ref{tab:ood_conf}).

\textbf{Per-dataset results for ECE.}
In Section~\ref{sec:experiments-how-ensemble}, we talked about the ECE of the standard and robust models \emph{after calibrating on ID data}.
Here we show the results for each dataset ID (Table~\ref{tab:id_ece}) and OOD (Table~\ref{tab:ood_ece}).
We also show the ECE of the standard and robust models \emph{before calibrating on ID data}, on ID (Table~\ref{tab:id_ece_precalibration}) and on OOD (Table~\ref{tab:ood_ece_precalibration}).

%%%%%%%%%%%%%% ECE, post-calibration (ID) %%%%%%%%%%%%%%%%%%%
\begin{table*}[t]
\caption{
\emph{ID} ECE: The expected calibration error (ECE) of the standard and robust models on ID test data, after post-calibration in ID validation data.
The ID calibration errors are low---note that we only use 500 examples to temperature scale, so for ImageNet we have fewer examples than classes for post-calibration, but the models are still fairly well calibrated.
}
\label{tab:id_ece}
\vskip 0.15in
\begin{center}
\begin{tabular}{cccccccc}
\toprule
 & Ent30 & DomNet & CIFAR10 & Liv17 & Land & Crop & CelebA\\
\midrule
Cal. Standard & 0.7 (0.1) & 2.0 (0.3) & 0.8 (0.2) & 1.3 (0.2) & 1.1 (0.5) & 1.4 (0.3) & 2.7 (0.4)\\
Cal. Robust & 1.1 (0.4) & 2.2 (0.2) & 1.3 (0.2) & 1.8 (0.0) & 1.7 (0.3) & 3.5 (0.2) & 1.2 (0.3)\\
\bottomrule
\end{tabular}
\vspace{1.2mm}
\newline
\begin{tabular}{cccccc}
\toprule
 & ImageNet & iWildCam & MNLI & Waterbirds & Comments\\
\midrule
Cal. Standard & 1.2 (-) & 3.6 (-) & 2.2 (-) & 1.2 (-) & 1.2 (-)\\
Cal. Robust & 2.3 (-) & 1.3 (-) & 2.5 (-) & 0.5 (-) & 8.1 (-)\\
\bottomrule
\end{tabular}
\end{center}
\vskip -0.1in
\end{table*}

%%%%%%%%%%%%%% ECE, post-calibration (OOD) %%%%%%%%%%%%%%%%%%%
\begin{table*}[t]
\caption{
\emph{OOD} ECE: The expected calibration error (ECE) of the standard and robust models on OOD test data, after calibrating on ID validation data.
The calibration errors here are high, especially compared to the ID calibration errors in Table~\ref{tab:id_ece}.
}
\label{tab:ood_ece}
\vskip 0.15in
\begin{center}
\begin{tabular}{cccccccc}
\toprule
 & Ent30 & DomNet & CIFAR10 & Liv17 & Land & Crop & CelebA\\
\midrule
Cal. Standard & 15.4 (0.8) & 13.6 (1.5) & 5.6 (1.1) & 11.4 (0.3) & 16.4 (0.8) & 7.4 (4.8) & 11.5 (1.0)\\
Cal. Robust & 14.3 (1.5) & 5.5 (0.5) & 8.2 (0.0) & 8.7 (0.2) & 6.5 (1.1) & 5.0 (0.3) & 14.0 (1.4)\\
\bottomrule
\end{tabular}
\vspace{1.2mm}
\newline
\begin{tabular}{cccccccc}
\toprule
 & ImNet-R & ImNet-V2 & ImNet-Sk & iWildCam & MNLI & Waterbirds & Comments\\
\midrule
Cal. Standard & 5.4 (-) & 4.0 (-) & 10.1 (-) & 3.2 (-) & 13.2 (-) & 17.7 (-) & 23.3 (-)\\
Cal. Robust & 4.0 (-) & 4.9 (-) & 5.1 (-) & 2.4 (-) & 4.2 (-) & 5.5 (-) & 6.3 (-)\\
\bottomrule
\end{tabular}

\end{center}
\vskip -0.1in
\end{table*}

%%%%%%%%%%%%%%% These should certainly go into the Appendix. %%%%%%%%%%%%%

%%%%%%%%%%%%%% Confidence on ID data (after calibration) %%%%%%%%%%%%%%%%%%%
\begin{table*}[t]
\caption{
\emph{ID} Confidences: The confidence of the standard and robust models on ID test data (after calibrating on ID data).
The standard model is typically more confidence than the robust model, which is reasonable since the standard model is also typically more accurate.
There are a few exceptions such as DomainNet, CelebA, and WaterBirds where the standard model is less confident than the robust model, but the standard model is also less accurate in these cases, so this is also reasonable.
}
\label{tab:id_conf}
\vskip 0.15in
\begin{center}
\begin{tabular}{cccccccc}
\toprule
 & Ent30 & DomNet & CIFAR10 & Liv17 & Land & Crop & CelebA\\
\midrule
Cal. Standard & 93.1 (0.3) & 83.7 (0.4) & 96.9 (0.6) & 97.0 (0.2) & 76.5 (0.9) & 95.5 (0.4) & 91.7 (0.6)\\
Cal. Robust & 89.9 (0.4) & 89.6 (0.1) & 91.0 (0.1) & 96.0 (0.1) & 71.3 (0.5) & 94.9 (0.5) & 94.7 (0.2)\\
\bottomrule
\end{tabular}
\vspace{1.2mm}
\newline
\begin{tabular}{cccccc}
\toprule
 & ImageNet & iWildCam & MNLI & Waterbirds & Comments\\
\midrule
Cal. Standard & 82.1 (-) & 82.1 (-) & 82.6 (-) & 87.9 (-) & 93.6 (-)\\
Cal. Robust & 68.1 (-) & 82.3 (-) & 81.9 (-) & 93.2 (-) & 87.0 (-)\\
\bottomrule
\end{tabular}
\end{center}
\vskip -0.1in
\end{table*}

%%%%%%%%%%%%%% Confidence on OOD data (after calibration) %%%%%%%%%%%%%%%%%%%
\begin{table*}[t]
\caption{
\emph{OOD} Confidences. The confidence of the standard and robust models on OOD test data (after calibrating on ID data).
The robust model is usually more confident than the standard model, which is reasonable since the robust model is also typically more accurate.
However, Landcover is a noticable exception: the robust model is less confident OOD, even though it is more accurate (see Table~\ref{tab:ood_results}).
}
\label{tab:ood_conf}
\vskip 0.15in
\begin{center}
\begin{tabular}{cccccccc}
\toprule
 & Ent30 & DomNet & CIFAR10 & Liv17 & Land & Crop & CelebA\\
\midrule
Cal. Standard & 76.1 (0.8) & 68.9 (1.5) & 87.8 (1.2) & 89.2 (0.5) & 72.0 (1.9) & 92.8 (1.0) & 85.5 (1.5)\\
Cal. Robust & 77.5 (0.4) & 92.6 (0.4) & 93.3 (0.1) & 90.8 (0.2) & 66.0 (0.6) & 94.1 (0.4) & 90.1 (0.1)\\
\bottomrule
\end{tabular}
\vspace{1.2mm}
\newline
\begin{tabular}{cccccccc}
\toprule
 & ImNet-R & ImNet-V2 & ImNet-Sk & iWildCam & MNLI & Waterbirds & Comments\\
\midrule
Cal. Standard & 57.8 (-) & 75.5 (-) & 50.6 (-) & 59.1 (-) & 77.0 (-) & 78.1 (-) & 80.1 (-)\\
Cal. Robust & 74.0 (-) & 64.2 (-) & 53.2 (-) & 65.1 (-) & 79.7 (-) & 92.5 (-) & 80.4 (-)\\
\bottomrule
\end{tabular}
\end{center}
\vskip -0.1in
\end{table*}

%%%%%%%%%%%%%% ECE, before calibration (ID) %%%%%%%%%%%%%%%%%%%
%%%%%%%%%%%%%% NOTE: this is before calibration %%%%%%%%%%%%%%%%%%%
\begin{table*}[t]
\caption{
\emph{ID} ECE. The expected calibration error (ECE) of the standard and robust models on ID test data, \emph{before calibration} (the key difference from Table~\ref{tab:id_ece} is that this is before calibration).
We can see that calibration on ID substantially reduces the ECE on ID data (see Table~\ref{tab:id_ece})
}
\label{tab:id_ece_precalibration}
\vskip 0.15in
\begin{center}
\begin{tabular}{cccccccc}
\toprule
 & Ent30 & DomNet & CIFAR10 & Liv17 & Land & Crop & CelebA\\
\midrule
Standard & 1.0 (0.1) & 8.5 (0.7) & 1.2 (0.1) & 1.2 (0.1) & 6.7 (1.2) & 1.5 (0.3) & 5.9 (0.5)\\
Robust & 1.1 (0.3) & 5.8 (1.3) & 1.1 (0.2) & 3.4 (0.4) & 1.3 (0.1) & 3.5 (0.1) & 1.8 (0.2)\\
\bottomrule
\end{tabular}
\vspace{1.2mm}
\newline
\begin{tabular}{cccccc}
\toprule
 & ImageNet & iWildCam & MNLI & Waterbirds & Comments\\
\midrule
Standard & 2.2 (-) & 10.9 (-) & 9.0 (-) & 8.2 (-) & 3.7 (-)\\
Robust & 2.4 (-) & 2.8 (-) & 8.2 (-) & 14.8 (-) & 10.2 (-)\\
\bottomrule
\end{tabular}
\end{center}
\vskip -0.1in
\end{table*}

%%%%%%%%%%%%%% ECE, before calibration (OOD) %%%%%%%%%%%%%%%%%%%
%%%%%%%%%%%%%% NOTE: this is before calibration %%%%%%%%%%%%%%%%%%%
\begin{table*}[t]
\caption{
\emph{OOD} ECE: The expected calibration error (ECE) of the standard and robust models on OOD test data, \emph{before calibration} (the key difference from Table~\ref{tab:ood_ece} is that this is before calibration).
The calibration errors here are higher than the ID calibration errors in Table~\ref{tab:id_ece_precalibration}.
Comparing with Table~\ref{tab:ood_ece} (which is after calibration on ID data), we see that calibrating ID does help OOD calibration a little, although the models still remain miscalibrated OOD.
}
\label{tab:ood_ece_precalibration}
\vskip 0.15in
\begin{center}
\begin{tabular}{cccccccc}
\toprule
 & Ent30 & DomNet & CIFAR10 & Liv17 & Land & Crop & CelebA\\
\midrule
Standard & 19.1 (0.3) & 29.5 (0.5) & 10.1 (0.3) & 11.7 (0.4) & 24.7 (1.5) & 8.3 (4.3) & 17.6 (0.5)\\
Robust & 14.3 (1.6) & 1.8 (0.8) & 8.4 (0.3) & 6.8 (0.2) & 7.1 (1.3) & 8.4 (0.7) & 12.7 (0.7)\\
\bottomrule
\end{tabular}
\vspace{1.2mm}
\newline
\begin{tabular}{cccccccc}
\toprule
 & ImNet-R & ImNet-V2 & ImNet-Sk & iWildCam & MNLI & Waterbirds & Comments\\
\midrule
Standard & 7.9 (-) & 6.1 (-) & 13.3 (-) & 19.5 (-) & 22.7 (-) & 31.8 (-) & 30.0 (-)\\
Robust & 3.9 (-) & 5.2 (-) & 5.2 (-) & 5.3 (-) & 10.3 (-) & 10.4 (-) & 9.9 (-)\\
\bottomrule
\end{tabular}
\end{center}
\vskip -0.1in
\end{table*}

\end{document}